\documentclass[11pt,a4paper]{article}
\usepackage[T1]{fontenc}
\usepackage[utf8]{inputenc}
\usepackage{lmodern}
\usepackage[margin=1in]{geometry}
\usepackage{microtype}
\usepackage{graphicx}
\usepackage{amsmath,amssymb,mathtools,amsthm}
\usepackage{booktabs,multirow,makecell,tabularx}
\usepackage{authblk}
\usepackage[table]{xcolor}
\usepackage[inline]{enumitem}
\usepackage[font=small,labelfont=bf]{caption}
\usepackage[authoryear,round]{natbib}
\usepackage{xurl}
\usepackage[hidelinks]{hyperref}
\usepackage{bookmark}
\usepackage[acronym]{glossaries}
\setlist[enumerate]{noitemsep,topsep=0pt}
\graphicspath{{figs/}}%

\newcommand{\card}[1]{|#1|}
\newcommand{\mat}[1]{\boldsymbol{#1}}
\newcommand{\note}[1]{\text{#1}}
\DeclareMathAlphabet{\mathcmcal}{OMS}{cmsy}{m}{n}
\newcommand{\set}[1]{\mathcmcal{#1}}
\DeclareMathOperator*{\argmax}{argmax}
\newtheorem{thm}{Theorem}[section]
\newcommand{\secref}[1]{Section~\ref{#1}}
\newcommand{\appref}[1]{Appendix~\ref{#1}}
\newcommand{\figref}[1]{Figure~\ref{#1}}
\newcommand{\tabref}[1]{Table~\ref{#1}}
\newcommand{\tabsref}[1]{Tables~\ref{#1}}
\newcommand{\eqnref}[1]{Eq.~(\ref{#1})}
\newcommand{\eqnsref}[1]{Eqs.~(\ref{#1})}
\newcommand{\thmref}[1]{Theorem~\ref{#1}}
\newacronym{gbs}{GBS}{Generalised Balanced Softmax}
\newacronym{ldam}{LDAM}{Label-Distribution-Aware Margin}
\newacronym{reslt}{ResLT}{Residual Learning for Long-tailed Recognition}
\newacronym{crt}{cRT}{Classifier Re-training}
\newacronym{sgd}{SGD}{Stochastic Gradient Descent}

\title{\texorpdfstring{
      \acrlong{gbs}: A Finite-Data Perspective on Logit Adjustment for Long-Tailed Recognition
   }{
      GBS: A Finite-Data Perspective on Logit Adjustment for Long-Tailed Recognition
   }
}

\author[1]{Yi-Hang Zhu}
\author[1]{Rajeev Raman}
\author[2,3]{Shiqi Su}
\author[4]{Jianyuan Sun}
\author[1]{Xinyu Yang}
\author[5]{Nan Xing}
\author[1]{Huiyu Zhou\thanks{Corresponding author: \href{mailto:hz143@leicester.ac.uk}{hz143@leicester.ac.uk}.}}
\affil[1]{School of Computing and Mathematical Sciences, University of Leicester, Leicester, LE1 7RH, UK}
\affil[2]{Department of Physics and Astronomy, University of Leicester, Leicester, LE1 7RH, UK}
\affil[3]{Scientific Computing, Rutherford Appleton Laboratory, Science and Technology Facilities Council, Didcot, OX11 0QX, UK}
\affil[4]{School of Computer Science and Informatics, De Montfort University, Leicester, LE1 9BH, UK}
\affil[5]{School of Automation and Information Engineering, Xi'an University of Technology, Xi'an, 710048, China}
\date{}
\hypersetup{
  pdftitle={Generalised Balanced Softmax: A Finite-Data Perspective on Logit Adjustment for Long-Tailed Recognition},
  pdfauthor={Yi-Hang Zhu, Rajeev Raman, Shiqi Su, Jianyuan Sun, Xinyu Yang, Nan Xing, Huiyu Zhou}
}

\begin{document}
\maketitle
\begin{abstract}
\noindent Models trained on long-tailed data using standard softmax tend to exhibit higher training error and a larger generalisation gap for classes with fewer training samples. We characterise this class-wise disparity as {\it the preference issue} and quantify it using a new metric, the model imbalance level $I$. To understand this issue, we analyse how imbalanced training data adversely affects class-wise gradients under standard softmax training. This paper then develops a finite-data \gls{gbs} framework for analysing and mitigating the preference issue. The framework uses the training-time logit adjustment $z_{nc}+\beta\log\card{\set{N}_c}$, which is algebraically identical to the training-time logit-adjusted loss of \citet{menon2021longtail} when $\tau=\beta$. The case $\beta=1$ also coincides with Balanced Softmax and with the unit adjustment supported by the Fisher-consistency argument under the true data distribution, corresponding to an idealised infinite-data setting. Building on this existing loss family, this paper uses a heuristic power-law assumption to motivate the adjustable coefficient and studies how $\beta$ affects trained models. Across the evaluated long-tailed benchmarks, $\beta=1$ does not attain the highest average testing recall on most datasets, showing that a different coefficient can be preferable when training on finite data. The selected values of $\beta$ reduce $I$ and improve average testing recall relative to the $\beta=1$ reference, while retaining negligible computational overhead and compatibility with existing representation-learning frameworks.
\end{abstract}
\glsresetall

\noindent\textit{Version note. This manuscript is a revised author version of
``Class-Balanced Softmax: A Bayes Theory-Based Method for Long-Tailed Recognition'',
published online in Pattern Recognition~\citep{zhu2026cbs},
DOI: \href{https://doi.org/10.1016/j.patcog.2026.114961}{10.1016/j.patcog.2026.114961}.
It adopts the name Generalised Balanced Softmax (GBS).
This author revision differs from the published journal version in its terminology
and theoretical presentation, including the relationship to existing logit-adjustment
methods; the underlying loss function and reported experimental results remain unchanged.}
\bigskip

\section{Introduction}
The potential of deep learning models was awakened by the ImageNet challenges~\citep{deng2009imagenet} held in 2012. Since then, deep learning models have achieved impressive performance in various computer vision tasks, including image classification~\citep{krizhevsky2012imagenet,he2016deep,dosovitskiy2021an} and image segmentation~\citep{he2017mask}. This paper focuses on image classification, which plays a crucial role in various tasks, including object detection~\citep{wang2023yolov7} and instance segmentation~\citep{he2017mask}.

In the image classification community, the datasets under investigation often suffer from imbalance, i.e., some classes have significantly more training samples than others~\citep{van2018inaturalist,gupta2019lvis}. Often, the training data exhibits a long-tailed distribution~\citep{van2018inaturalist,gupta2019lvis}.
Following the existing literature, we use the terms {\it head classes}, {\it medium classes} and {\it tail classes} to split the classes in a given dataset into three groups. Head classes have the largest number of training samples per class, while tail classes have the lowest.
It has been observed that models trained on imbalanced data tend to provide a higher recall on the testing data for head classes compared to tail classes~\citep{zhang2023deep}, referred to as the {\it imbalanced testing performance}.
This paper focuses on imbalanced classification. The goal is to mitigate the impact of imbalanced data on the model training and ensure the model achieves high accuracy for all classes on the testing data.

To handle imbalanced classification, a variety of methods have been proposed in the existing literature~\citep{yang2022survey,zhang2025systematic}.
In general, the methods can be categorised into four types.
\begin{enumerate*}
   [label={(\arabic*)}]
   \item {Rebalancing methods}~\citep{cao2019learning,Ren2020balms,zhang2025class}: rebalancing methods aim to mitigate the impact of the imbalanced training data on model training. This is the key issue to be addressed in the field of imbalanced learning. As we discuss in~\secref{sec:related_work}, the existing rebalancing methods often rely on restrictive theoretical assumptions or reduce training data diversity, which limits their performance.
   \item {Representation learning methods}:
   these methods focus on designing the model structure or loss function to enable the model to use the available data more effectively for training and thus improve the quality of representations.
   This will improve the model performance for all the classes. It is particularly important for tail classes, which have few training samples.
   The strategies adopted vary widely, including ensemble architectures~\citep{wang2020long,tan2024ncl}, contrastive learning~\citep{cui2023generalized}, latent space enhancement~\citep{liu2024lcreg,hu2026parameter} and incorporating large vision language models~\citep{shi2026lift+}.
   However, these methods still rely on rebalancing methods as their classifier~\citep{wang2020long,cui2023generalized,tan2024ncl,shi2026lift+}.
   \item {Data augmentation}:  these methods aim to increase the diversity of training data and are commonly applied as a default setup during image preprocessing. Representative methods are the mixup-based approaches~\citep{zhang2018mixup,xiang2023margin,baik2024dbn}, which \citet{jian2026currmix} recently enhanced with a curriculum strategy that alternates between easier and harder samples during mixing.
   \item {Fixed linear classifiers}~\citep{yang2022inducing,peifeng2023feature}: the term {\it classifier} here refers to the last linear layer in the deep learning models~\citep{Kang2020Decoupling}. The methods of this type generate values for the linear classifier weights and fix them to avoid the models expressing their preference towards head classes via the linear classifier weights. However, it is challenging to generate high-quality values for the weights~\citep{yang2022inducing}.
\end{enumerate*}

The contributions of this paper are as follows.
\begin{itemize}
   \item We propose a finite-data \gls{gbs} framework for analysing and mitigating the class-wise preference exhibited by models trained on long-tailed datasets.
   The framework uses a coefficient $\beta\geq0$ in the training-time adjustment $z_{nc}+\beta\log\card{\set{N}_c}$, where $z_{nc}$ is the logit for training sample $n$ and class $c$, and $\card{\set{N}_c}$ is the number of training samples in class $c$. For any $\beta\geq0$, this loss is algebraically identical to the training-time logit-adjusted loss of~\citet{menon2021longtail} with $\tau=\beta$; when $\beta=1$, it also coincides with Balanced Softmax~\citep{Ren2020balms} and the unit adjustment supported by the Fisher-consistency argument of~\citet{menon2021longtail}. The Fisher-consistency argument uses the true data distribution, corresponding to an idealised infinite-data setting, while our framework complements it by examining the role of $\beta$ when a model is trained on finite data.
   \item To the best of our knowledge, this work is the first to characterise {\it the preference issue} as an inherent model symptom that provides a more fundamental insight into the aforementioned imbalanced testing performance of models trained on imbalanced data. This issue manifests as a disproportionately higher training error and larger generalisation gap for tail classes compared to head classes.
   To quantify this issue, we formulate a novel metric, referred to as {\it model imbalance level} ($I$).
   Furthermore, we explain how imbalanced training data adversely affects the class-wise gradients under standard softmax training through theoretical analysis and empirical validation.
   \item We evaluate \gls{gbs} on long-tailed CIFAR, ImageNet-LT, Places-LT, iNaturalist2018 and an LVIS-based dataset. Among the evaluated values of $\beta$, $\beta=1$ does not attain the highest average testing recall on most of the studied datasets. This finding shows that a different coefficient can be preferable when training on finite data, even though the unit adjustment is derived from the true data distribution under the idealised infinite-data setting. Compared with the common $\beta=1$ baseline, the reported \gls{gbs} results have a lower model imbalance level $I$ and higher average testing recall, while \gls{gbs} remains compatible with the BCL~\citep{zhu2022balanced} and ProCo~\citep{du2024probabilistic} representation-learning frameworks.
\end{itemize}

This paper is organised as follows. \secref{sec:related_work} discusses the representative existing rebalancing methods. \secref{sec:ib_train_data_impact} studies the impact of imbalanced training data on the gradients and model performance. \secref{sec:gbs} presents our \gls{gbs}. \secref{sec:experiments} provides experimental results, which confirm our analysis and demonstrate the performance of our \gls{gbs}. \secref{sec:conclusion} concludes the work.

\section{Related work}~\label{sec:related_work}
Rebalancing methods aim to mitigate the impact of the sample number difference between classes on the model.
The existing rebalancing methods mainly include three types: resampling-based methods, training-time loss-modification methods and post-hoc correction methods.
Both resampling-based methods and training-time loss-modification methods are applied during model training, while post-hoc correction methods are applied only at inference.

\paragraph{Resampling-based methods}
When training a model, samples are first selected, and then the selected samples are fed into the model for training.
Resampling-based methods aim to balance the number of samples selected for each class.
Undersampling and oversampling are the two classical resampling methods~\citep{zhou2005training}.
Undersampling is no longer attractive due to inefficient training data use.
By contrast, oversampling remains popular and was referred to as class-balanced sampling by~\citet{Kang2020Decoupling}. It is simple to implement and provides decent performance.
\citet{Kang2020Decoupling} also discussed different sampling methods and proposed the decoupled training framework, which has been applied in many later studies~\citep{yang2022survey}.
The \gls{crt} proposed by~\citet{Kang2020Decoupling} combines class-balanced sampling and decoupled training.
Resampling-based methods rebalance the training data at the expense of reducing data diversity. However, as demonstrated by~\citet{Kang2020Decoupling}, maintaining maximum data diversity is essential for achieving high-quality representation learning in long-tailed recognition. By reducing training data diversity, resampling-based methods often compromise the learned representations and limit overall performance.

\paragraph{Training-time reweighting methods}
Training-time reweighting methods are designed by incorporating weights into the loss function. These weights can be assigned per sample, per class, or even per group of classes. Focal Loss~\citep{lin2017focal} is one of the earliest reweighting-based methods, with weights calculated for each sample based on its softmax output. Samples with lower prediction probabilities are assigned higher weights. Consequently, this approach is often referred to as hard sample mining. \citet{son2025dbm} combine the hard mining idea with a class-wise margin in a difficulty-aware balancing margin loss. While tail classes typically contain a higher proportion of hard samples, head classes also contain a significant number of them, which limits the rebalancing effect of difficulty-based weighting alone.

Many training-time reweighting methods determine their weights from the number of training samples per class. Representative works include~\citep{Kang2020Decoupling,cao2019learning,Ren2020balms}. \citet{cao2019learning} proposed \gls{ldam}. However, as noted by~\citet{Ren2020balms}, \gls{ldam} is derived under a binary classification setup and may not be inherently suited for multi-class classification.
Balanced Softmax~\citep{Ren2020balms} remains a widely used rebalancing method~\citep{zhu2022balanced,du2024probabilistic,shi2026lift+}. It incorporates the class-wise training sample counts into the softmax normalisation. When the empirical class prior is used, Balanced Softmax is the $\beta=1$ case of \gls{gbs}. It is also equivalent to the training-time logit-adjusted loss of~\citet{menon2021longtail} when their coefficient $\tau$ is set to one. More generally, the training-time loss proposed by~\citet{menon2021longtail} with $\tau>0$ is algebraically identical to \gls{gbs} with $\beta=\tau$. Their Fisher-consistency argument justifies the unit adjustment using the true data distribution, corresponding to an idealised infinite-data setting. This paper therefore does not claim the parameterised loss itself as novel; instead, it develops a finite-data framework for studying the choice of the training-time coefficient through the preference issue, the model imbalance level, class-wise gradients and the relationship between class frequency and predicted probability. Balanced Softmax serves as the principal empirical baseline because it is the widely adopted $\beta=1$ instance of this logit-adjustment loss family.

Moreover, several studies also explore the limitations of the reliance on the class-wise sample count as the sole rebalancing signal. \citet{cui2019class} argued that samples within a class carry overlapping information and proposed the effective number of samples in place of the raw count. This formulation has recently been refined to admit partial overlap between samples~\citep{zhang2026e2igb}. Despite these developments, the class-wise sample count remains the dominant signal in practice, as the continued adoption of Balanced Softmax indicates.
More recently, \citet{cortes2025improved} embedded the losses underlying Balanced Softmax and logit adjustment~\citep{Ren2020balms,menon2021longtail} within a broader surrogate family and analysed their consistency properties. This loss-level theoretical perspective complements the finite-data analysis in this paper.

Beyond per-class reweighting, \citet{cui2023reslt} proposed \gls{reslt}, which assigns higher weights to tail classes in the loss function and incorporates the post-hoc $\tau$-norm correction~\citep{Kang2020Decoupling}. While \gls{reslt} achieves competitive results, it requires a significant number of additional trainable parameters and structural modifications to adapt to datasets with different imbalance levels.

\paragraph{Post-hoc correction methods} 
The $\tau$-norm method, proposed by~\citet{Kang2020Decoupling}, is based on the observation that models trained on imbalanced data exhibit larger norms for linear classifier weights corresponding to head classes than those for tail classes. To address this, $\tau$-norm normalises the classifier weights using a hyperparameter $\tau$. A key limitation of this approach is that these normalised weights are often far from the optimal configuration for the learned features; consequently, the performance gains remain limited. \citet{menon2021longtail} also proposed a post-hoc form of logit adjustment. For a model trained using standard softmax, it subtracts $\tau\log\pi_c$ from the logit for class $c$, where $\pi_c$ is the empirical training proportion of class $c$ and $\tau>0$ controls the adjustment strength. They varied $\tau$ in their finite-data experiments and explained that values other than one can compensate for imperfectly calibrated model probabilities. This post-hoc procedure is distinct from their training-time loss discussed above. The \gls{gbs} framework focuses on the behaviour of models trained with the training-time loss rather than on post-hoc correction.
Given that post-hoc adjustment is applied after training, it changes the decision boundary without modifying the learned representation. Its effectiveness therefore depends on the quality and calibration of the logits produced by the trained model.

Post-hoc logit adjustment has also been extended beyond image classification. For long-tailed object detection, \citet{alexandridis2025fractal} proposed a training-free post-calibration method that augments frequency-based logit adjustment with spatial information, using the fractal dimension of each class's location distribution to debias detections in both the class and spatial axes. Nevertheless, object detection is beyond the scope of this work.


\section{Impact of imbalanced training data on model training}\label{sec:ib_train_data_impact}
In a standard deep learning model for image classification, a backbone, such as ResNet~\citep{he2016deep}, maps each training sample $n$ to a feature vector $\mat{h}_{n}$, also referred to as a representation~\citep{bengio2013representation}. A linear layer, commonly called the linear classifier~\citep{Kang2020Decoupling}, then maps this feature vector to a class-specific score via
\begin{equation}
   \label{eq:linear_layer}
   z_{nc} = \sum_{i} w_{ci}h_{ni} + b_{c} \quad \forall n \in \set{N}, c\in \set{C},
\end{equation}
where $h_{ni}$ is the $i$th element of $\mat{h}_{n}$, $w_{ci}$ is the classifier weight connecting feature dimension $i$ to class $c$, and $b_c$ is the bias for class $c$. The sets $\set{N}$ and $\set{C}$ contain all training samples and all classes, respectively. The score $z_{nc}$ is called the logit for sample $n$ and class $c$, and the predicted class is the class index returned by
\begin{equation}
   \label{eq:classification}
   \argmax_{c}z_{nc}.
\end{equation}
For model training, the softmax probability $p_{nc}$ and cross-entropy loss $l^{\note{ce}}_{n}$ for sample $n$ are calculated as
\begin{equation}
   \label{eq:softmax}
   p_{nc} = \frac{e^{z_{nc}}}{\sum_{d\in \set{C}} e^{z_{nd}}},
\end{equation}
\begin{equation}
   \label{eq:ce_n}
   l^{\note{ce}}_{n} = -\sum_{c\in \set{C}} y_{nc}\log p_{nc},
\end{equation}
$y_{nc} \in \{0,1\}$ denotes the class indicator, where $y_{nc} = 1$ if sample $n$ belongs to class $c$, and $y_{nc} = 0$ otherwise.
Minimising the average cross-entropy loss over the training samples is equivalent to maximum-likelihood estimation.

During the back-propagation process, the linear classifier weights are updated via
\begin{gather}
   w_{ci}
   \leftarrow
   w_{ci} +
   \alpha \frac{\Delta^{\note{grad}}_{ci}}{\card{\set{N}}}\label{eq:w_update},
\end{gather}
where $\card{\set{N}}$ is the number of training samples, $\alpha$ is the learning rate, and $\Delta^{\note{grad}}_{ci}$ is the update signal calculated as
\begin{gather}
   \Delta^{\note{grad}}_{ci}
   = \sum_{n\in \set{N}_{c}}(1-p_{nc}){h}_{ni} - \sum_{d\in \set{C}\backslash \{c\}}\sum_{n\in \set{N}_{d}} p_{nc}{h}_{ni} \label{eq:delta_grad}
\end{gather}
The gradient of $w_{ci}$ is $-\frac{\Delta^{\note{grad}}_{ci}}{\card{\set{N}}}$. The sets $\set{N}_{c}$ and $\set{N}_{d}$ contain the training samples belonging to classes $c$ and $d$, respectively.
The expressions in \eqnsref{eq:linear_layer}--(\ref{eq:delta_grad}) are standard~\citep{goodfellow2016deep}.
In practice, deep learning models are trained using mini-batches~\citep{goodfellow2016deep}; for simplicity, the equations above are written using the full training set.

\paragraph{Imbalanced gradients}
In \eqnref{eq:delta_grad}, we call the first sum the \emph{reward} and the second sum the \emph{penalty}. For class $c$, the reward is contributed by samples from class $c$, whereas the penalty is contributed by samples from all other classes. The following simplified analysis shows how unequal class sizes create imbalanced class-wise update signals at initialisation.
\begin{thm}\label{thm:ib_grad}
    Let $\card{\set{C}}$ be the number of classes and $\card{\set{N}_{c}}$ be the number of training samples in class $c$. Consider a simplified initial state in which
    \begin{equation}
       p_{nc} = \frac{1}{\card{\set{C}}}\quad \forall n,c
    \end{equation}
    and
    \begin{equation}
       h_{ni} = \upsilon\quad\forall n, i
    \end{equation}
    where $\upsilon>0$ is a constant.
    Then we have
    \begin{gather}
       \Delta^{\note{grad}}_{ci}
       \propto  \card{\set{C}}\card{\set{N}_{c}}   - \card{\set{N}}.
    \end{gather}
\end{thm}
The proof for \thmref{thm:ib_grad} is provided in \appref{app:ib_grad}.
Under the conditions of the theorem, $\Delta^{\note{grad}}_{ci}$ increases linearly with $\card{\set{N}_c}$. Therefore, an imbalanced training set gives head classes larger initial update signals than tail classes. This result isolates one direct effect of unequal class sizes; it does not describe the complete optimisation trajectory. The class-wise reward and penalty signals observed during training are examined empirically in \secref{sec:better_gradients}.




\paragraph{The preference issue}
In the studied settings, models trained on imbalanced data exhibit higher training error and larger relative generalisation gaps for tail classes than for head classes. We refer to this class-wise disparity as the \emph{preference issue}. High training error indicates underfitting, as reflected by low training recall, while the relative generalisation gap measures the percentage decrease from training recall to testing recall. \figref{fig:preference_issue} illustrates this pattern using a ResNet50~\citep{he2016deep} with a standard softmax classifier on ImageNet-LT; the experimental configuration is provided in \secref{sec:setup}. The head classes maintain high recall on both the training and test sets, whereas the tail-class recall on the test set is nearly half its training value.
\begin{figure}[!tb]
    \centering
    \includegraphics[width=\textwidth]{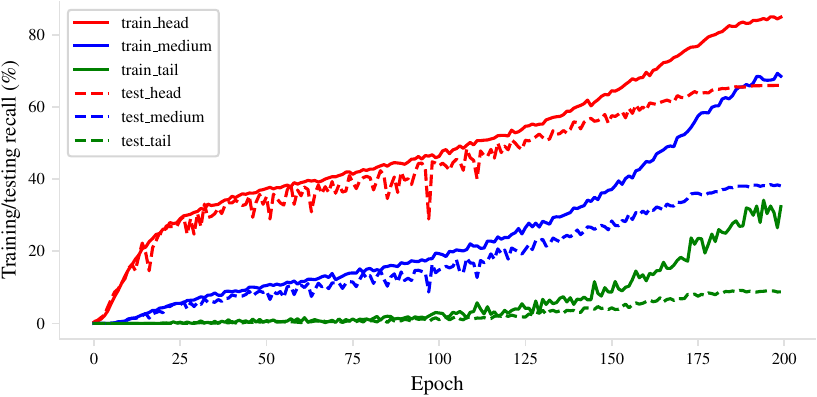}
    \caption{
    The preference issue observed in the model trained on ImageNet-LT: higher training error and larger generalisation gaps for tail classes. Solid lines represent average recall across class groups on the training data. Dashed lines represent testing recall.}
    \label{fig:preference_issue}
\end{figure}

Existing literature often attributes the poor performance of deep learning models for long-tailed recognition to the discrepancy between imbalanced training and an assumed balanced testing distribution~\citep{Ren2020balms}. While imbalanced training data is the root cause of performance degradation, we find the issue to be fundamentally independent of the testing data distribution, which in many real-world scenarios, such as the LVIS dataset~\citep{gupta2019lvis}, remains unknown and is itself often imbalanced.
To the best of our knowledge, this is the first work to reveal the preference issue as an inherent model characteristic caused by training on imbalanced data. This preference issue manifests as a structural imbalance in the model’s internal behaviour and persists regardless of the testing distribution. We believe that the preference issue provides a more fundamental insight into the nature of the long-tailed recognition problem. 

We introduce the \emph{model imbalance level} to quantify the preference issue.
Let $R^{\note{tr}}_{g}$ and $R^{\note{te}}_{g}$ denote the average recall, reported as a percentage, of class group $g \in \{h, m, t\}$ (head, medium, and tail) on the training and test sets, respectively. Assuming $R^{\note{tr}}_{g}>0$, the relative generalisation gap for group $g$ is defined as
\begin{equation}
    G_{g} = \frac{100 \times (R^{\note{tr}}_{g} - R^{\note{te}}_{g})}{R^{\note{tr}}_{g}}.
\end{equation}
The group-wise preference score, which combines training error and the relative generalisation gap, is defined as
\begin{equation}
    P_{g} = (100 - R^{\note{tr}}_{g}) + G_{g}.
\end{equation}
A lower $P_{g}$ indicates better group-specific performance. We then define the model imbalance level $I$ as the range of the three group-wise preference scores:
\begin{equation}
    I = \max(P_{h}, P_{m}, P_{t}) - \min(P_{h}, P_{m}, P_{t}).
\end{equation}
A lower $I$ indicates less disparity among the head, medium and tail groups. It should be interpreted together with the overall recall, because a low value of $I$ does not by itself imply high predictive performance. $I$ uses both training and test data because $G_g$ depends on testing recall $R^{\note{te}}_g$. Since $R^{\note{te}}_g$ is computed from class-wise recalls, $I$ can also be computed when the test set is imbalanced.
\paragraph{Remark}
Aggregating these quantities over class groups ($h$, $m$ and $t$) rather than individual classes improves numerical stability because an individual class may have zero training recall even when its group-average training recall is positive. Nevertheless, $G_g$ and $I$ require a positive training recall for each group.

\section{\texorpdfstring{\acrlong{gbs}}{GBS}}\label{sec:gbs}
This section develops the finite-data interpretation of \gls{gbs}. We first show how class priors enter conditional class probabilities, then examine averages of model-predicted probabilities over a finite training set, and finally motivate the coefficient $\beta$ used in the training-time loss.

\paragraph{Class priors and conditional class probabilities}
\begin{thm}
   \label{thm:two_p}
   Let $x_n$ denote the input of sample $n$. Let $\pi_c>0$ and $\hat{\pi}_c>0$ denote two class priors for class $c$, with $\sum_{c\in\set{C}}\pi_c=\sum_{c\in\set{C}}\hat{\pi}_c=1$. These are unconditional probabilities and therefore do not depend on sample $n$. Assume that the two settings share the same class-conditional density $p(x_n\mid y=c)$. Their conditional class probabilities for sample $n$ are
   \begin{equation}
       \phi_{nc}
       = p_{\pi}(y=c\mid x_n)
       = \frac{p(x_n\mid y=c)\pi_c}{\sum_{d\in\set{C}}p(x_n\mid y=d)\pi_d}
   \end{equation}
   and
   \begin{equation}
       \hat{\phi}_{nc}
       = p_{\hat{\pi}}(y=c\mid x_n)
       = \frac{p(x_n\mid y=c)\hat{\pi}_c}{\sum_{d\in\set{C}}p(x_n\mid y=d)\hat{\pi}_d}.
   \end{equation}
   Thus, $\phi_{nc}$ and $\hat{\phi}_{nc}$ depend on both the sample and the class. If $\phi_{nc}$ is represented by the standard softmax probability in \eqnref{eq:softmax}, then $\hat{\phi}_{nc}$ can be written as
   \begin{equation}
      \label{eq:prior_adjusted_softmax}
      \hat{\phi}_{nc} = \frac{\frac{\hat{\pi}_c}{\pi_c}e^{z_{nc}}}{\sum_{d\in \set{C}} \frac{\hat{\pi}_d}{\pi_d}e^{z_{nd}}},
   \end{equation}
   with the corresponding decision rule
   \begin{equation}
       \argmax_{c} \left(z_{nc} + \log \frac{\hat{\pi}_c}{\pi_c}\right).
   \end{equation}
\end{thm}
The formal proof for \thmref{thm:two_p} is provided in \appref{app:proof_thm_two_p}.
The theorem shows how changing the class prior from $\pi_c$ to $\hat{\pi}_c$ transforms the conditional class probability for each sample while the class-conditional density remains fixed. Here, $\pi_c$ and $\hat{\pi}_c$ denote unconditional class priors, whereas $\phi_{nc}$ and $\hat{\phi}_{nc}$ denote the corresponding conditional probabilities for sample $n$. This prior-adjustment identity underlies Balanced Softmax and logit adjustment~\citep{Ren2020balms,menon2021longtail}.

\paragraph{Model-induced class marginals}
For the finite-data analysis, $\phi_{nc}$ denotes the conditional probability predicted by a balanced-data reference model, and $\hat{\phi}_{nc}$ denotes the corresponding probability predicted by a model trained on imbalanced data using standard softmax. We average these conditional probabilities over the same training set to define the model-induced class marginals
\begin{equation}\label{eq:p_y=c_total}
    \bar{\phi}_c = \frac{1}{\card{\set{N}}}\sum_{n\in \set{N}}\phi_{nc},
\end{equation}
and
\begin{equation}\label{eq:hat_p_y=c_total}
    \bar{\hat{\phi}}_c = \frac{1}{\card{\set{N}}}\sum_{n\in \set{N}}\hat{\phi}_{nc}.
\end{equation}
Unlike the class priors $\pi_c$ and $\hat{\pi}_c$ in \thmref{thm:two_p}, $\bar{\phi}_c$ and $\bar{\hat{\phi}}_c$ are averages of model outputs. They are not assumed to equal the true class priors. If a model made perfectly correct one-hot predictions on every training sample, its induced class marginal would equal the empirical class frequency $\card{\set{N}_{c}}/\card{\set{N}}$. A trained neural network instead produces soft probabilities, so its induced marginal can deviate from this frequency according to its class-wise behaviour.

Motivated by the class-wise ratio in \eqnref{eq:prior_adjusted_softmax}, the finite-data framework uses the empirical adjustment factor $a_c=\bar{\hat{\phi}}_c/\bar{\phi}_c$. This is a heuristic construction from model-induced marginals, not an identification of those marginals with the true class priors. Applying $a_c$ inside the softmax gives the adjusted training probability
\begin{equation}
   \label{eq:general_balanced_softmax}
   q_{nc} = \frac{\frac{\bar{\hat{\phi}}_c}{\bar{\phi}_c}e^{z_{nc}}}{\sum_{d\in \set{C}} \frac{\bar{\hat{\phi}}_d}{\bar{\phi}_d}e^{z_{nd}}}.
\end{equation}

\begin{thm}\label{thm:p_y=c_pattern}
    Suppose a class-balanced reference model assigns probability $1-\epsilon$ to the target class of every training sample and distributes the remaining probability uniformly among the other $\card{\set{C}}-1$ classes, where $0<\epsilon<(\card{\set{C}}-1)/\card{\set{C}}$. Then $\bar{\phi}_c < \frac{\card{\set{N}_{c}}}{\card{\set{N}}}$ when $\card{\set{N}_{c}}\card{\set{C}} > \card{\set{N}}$, and $\bar{\phi}_c \geq \frac{\card{\set{N}_{c}}}{\card{\set{N}}}$ otherwise.
\end{thm}
The proof for \thmref{thm:p_y=c_pattern} is provided in \appref{app:p_y=c_proof}.
This stylised result predicts how the induced class marginal of an equal-confidence reference model differs from the empirical class frequency. It is not a claim about every class-balanced model. The empirical results in \figref{fig:p_b_c} and \tabref{tab:p_b_c} examine whether the predicted sign pattern is observed in the studied models.
In \figref{fig:p_b_c}, $\bar{\phi}_c$ is computed on the ImageNet-LT training data using a model trained on the balanced ImageNet-1K dataset. We use this model as a balanced-data reference.
\begin{figure}[!tb]
   \centering
   \includegraphics[width=\textwidth]{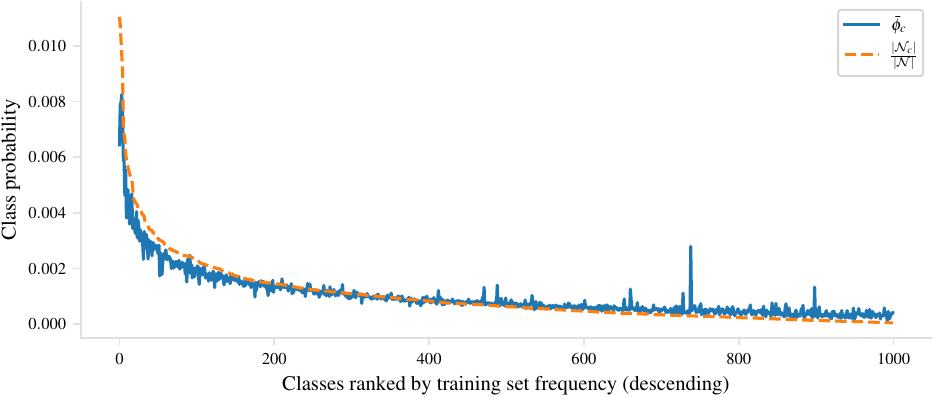}
   \caption{Model-induced class marginal generated by a balanced-data reference model on the ImageNet-LT training data. The observed $\bar{\phi}_c$ is lower than the class frequency $\frac{\card{\set{N}_{c}}}{\card{\set{N}}}$ for head classes and higher for tail classes, consistent with the sign pattern predicted by \thmref{thm:p_y=c_pattern}.}
   \label{fig:p_b_c}
\end{figure}
In \tabref{tab:p_b_c}, each value is the average relative difference between $\bar{\phi}_c$ and $\frac{\card{\set{N}_{c}}}{\card{\set{N}}}$ across head, medium and tail classes, calculated on the training data of various datasets.
The relative difference is calculated via 
\begin{equation}
    {\Delta\bar{\phi}} = \frac{\bar{\phi}_c-\frac{\card{\set{N}_{c}}}{\card{\set{N}}}}{\frac{\card{\set{N}_{c}}}{\card{\set{N}}}}\times100\%.
\end{equation}
The reference models are trained on ImageNet-1K or the full, balanced CIFAR10 or CIFAR100 datasets.

\begin{table}[!t]\centering\small
\caption{Relative difference between the model-induced class marginal and class frequency $\frac{\card{\set{N}_{c}}}{\card{\set{N}}}$ for a
balanced-data reference model. Consistent with the sign pattern predicted by \thmref{thm:p_y=c_pattern},
$\bar{\phi}_c$ is lower than $\frac{\card{\set{N}_{c}}}{\card{\set{N}}}$ for head classes
and higher for tail classes. Values represent the average
relative difference $\Delta\bar{\phi}$.}\label{tab:p_b_c}
\begin{tabular}{lrrrr}\toprule
Dataset &Head &Medium &Tail \\\midrule
C10-LT10 &-7.86 &3.85 &34.37 \\
C10-LT50 &-14.84 &22.13 &241.43 \\
C10-LT100 &-17.21 &39.93 &501.08 \\
C100-LT10 &-10.13 &9.06 &46.77 \\
C100-LT50 &-12.12 &31.63 &171.11 \\
C100-LT100 &-12.09 &48.48 &300.33 \\
ImageNet-LT &-9.79 &47.48 &280.27 \\
\bottomrule
\end{tabular}
\end{table}

\tabref{tab:p_hat_dist} reports the corresponding induced marginal $\bar{\hat{\phi}}_c$ for models trained on imbalanced data using standard softmax. In the studied datasets, $\bar{\hat{\phi}}_c > \frac{\card{\set{N}_{c}}}{\card{\set{N}}}$ for head classes and $\bar{\hat{\phi}}_c < \frac{\card{\set{N}_{c}}}{\card{\set{N}}}$ for tail classes, matching the sign pattern analysed in \appref{app:hat_p_y=c_proof}.
Each value in \tabref{tab:p_hat_dist} is the average relative difference between $\bar{\hat{\phi}}_c$ and $\frac{\card{\set{N}_{c}}}{\card{\set{N}}}$ across the classes in the indicated group.
The relative difference is calculated via 
\begin{equation}
    \Delta\bar{\hat{\phi}} = \frac{\bar{\hat{\phi}}_c-\frac{\card{\set{N}_{c}}}{\card{\set{N}}}}{\frac{\card{\set{N}_{c}}}{\card{\set{N}}}}\times100\%.
\end{equation}
The detailed experimental setup for \figref{fig:p_b_c} and \tabsref{tab:p_b_c} and~\ref{tab:p_hat_dist} is introduced in \secref{sec:setup}.

\begin{table}[!t]
   \centering\small
   \caption{Relative difference between the model-induced class marginal and class frequency $\frac{\card{\set{N}_{c}}}{\card{\set{N}}}$ for
   an imbalanced model. For models trained with standard softmax, 
   $\bar{\hat{\phi}}_c$ is higher than $\frac{\card{\set{N}_{c}}}{\card{\set{N}}}$ for head
   classes and lower for tail classes. Values represent
   the average relative difference $\Delta\bar{\hat{\phi}}$.
   }\label{tab:p_hat_dist}
   \begin{tabular}{lrrrr}
      \toprule
      Dataset & Head &Medium &Tail    \\\midrule
      C10-LT10   & 1.13 & -0.31  & -5.28 \\
      C10-LT50   & 1.11 & -1.70 & -16.48 \\
      C10-LT100  & 0.84 & -2.50 & -20.10 \\
      C100-LT10     & 6.39 & -5.30   & -29.40  \\
      C100-LT50     & 4.81 & -12.48  & -58.94 \\
      C100-LT100        & 3.84 & -17.65   & -71.61  \\
      ImageNet-LT & 0.49 & -1.51  & -12.93  \\
      \bottomrule
   \end{tabular}
\end{table}

\paragraph{Power-law approximation of the class marginals}

To describe the deviations of $\bar{\phi}_c$ and $\bar{\hat{\phi}}_c$ from the empirical class frequency, we introduce positive class-specific scaling factors $\kappa_c$ and $\hat{\kappa}_c$:
\begin{gather}
   \bar{\phi}_c = \frac{\card{\set{N}_{c}}}{\card{\set{N}}}\kappa_{c}/\sum_{c'\in \set{C}}\frac{\card{\set{N}_{c'}}}{\card{\set{N}}}\kappa_{c'}\\
   \bar{\hat{\phi}}_c = \frac{\card{\set{N}_{c}}}{\card{\set{N}}}\hat{\kappa}_{c}/ \sum_{c'\in \set{C}}\frac{\card{\set{N}_{c'}}}{\card{\set{N}}}\hat{\kappa}_{c'}.
\end{gather}
The patterns in \tabref{tab:p_b_c} indicate that $\kappa_c$ should generally decrease with class frequency, whereas those in \tabref{tab:p_hat_dist} indicate that $\hat{\kappa}_c$ should generally increase with class frequency. We model these trends heuristically using the power-law forms
\begin{gather}
   \kappa_{c} = \card{\set{N}_{c}}^{\beta^{\note{b}}} \label{eq:kappa}\\
   \hat{\kappa}_{c} = \card{\set{N}_{c}}^{\beta^{\note{ib}}}\label{eq:hat_kappa},
\end{gather}
where $\beta^{\note{b}}\leq 0$ and $\beta^{\note{ib}}\geq 0$. These power laws are parsimonious monotonic approximations of the observed relationships; they are not asserted to be universal or analytically optimal.

Substituting these approximations into \eqnref{eq:general_balanced_softmax} gives
\begin{gather}
   q_{nc}
   \approx \frac{\frac{\card{\set{N}_{c}}^{\beta^{\note{ib}}}/ \sum_{c'}\frac{\card{\set{N}_{c'}}}{\card{\set{N}}}\card{\set{N}_{c'}}^{\beta^{\note{ib}}}}{\card{\set{N}_{c}}^{\beta^{\note{b}}}/\sum_{c'}\frac{\card{\set{N}_{c'}}}{\card{\set{N}}}\card{\set{N}_{c'}}^{\beta^{\note{b}}}}e^{z_{nc}}}{\sum_{d\in \set{C}} \frac{\card{\set{N}_{d}}^{\beta^{\note{ib}}}/ \sum_{d'}\frac{\card{\set{N}_{d'}}}{\card{\set{N}}}\card{\set{N}_{d'}}^{\beta^{\note{ib}}}}{\card{\set{N}_{d}}^{\beta^{\note{b}}}/\sum_{d'}\frac{\card{\set{N}_{d'}}}{\card{\set{N}}} \card{\set{N}_{d'}}^{\beta^{\note{b}}} } e^{z_{nd} }}
   =
   \frac{
       \card{\set{N}_{c}}^{\beta^{\note{ib}}-\beta^{\note{b}}} e^{z_{nc}}
   }{
       \sum_{d\in \set{C}} \card{\set{N}_{d}}^{\beta^{\note{ib}}-\beta^{\note{b}}} e^{z_{nd}}
   } 
\end{gather}
Defining $\beta=\beta^{\note{ib}}-\beta^{\note{b}}\geq0$, we refer to the resulting parameterisation as \gls{gbs}, with the following conditional probability used during training:
\begin{equation}\label{eq:gbs}
   q_{nc} = \frac{\card{\set{N}_{c}}^{\beta} e^{z_{nc}}}{\sum_{d\in \set{C}} \card{\set{N}_{d}}^{\beta} e^{z_{nd}}}.
\end{equation}

The \gls{gbs} loss is obtained by substituting $q_{nc}$ from \eqnref{eq:gbs} for $p_{nc}$ in the cross-entropy loss in \eqnref{eq:ce_n}. \eqnref{eq:gbs} adds $\beta\log\card{\set{N}_c}$ to each class logit inside the training-time softmax. It is algebraically identical to the training-time logit-adjusted loss proposed by~\citet{menon2021longtail}. Specifically, using the empirical class prior $\pi_c=\card{\set{N}_c}/\card{\set{N}}$ and setting the coefficient $\tau=\beta$ gives
\begin{equation}
   \pi_c^{\beta}e^{z_{nc}}
   = (\frac{\card{\set{N}_c}}{\card{\set{N}}})^{\beta} e^{z_{nc}},
\end{equation}
where the class-independent factor $\card{\set{N}}^{-\beta}$ cancels in the softmax normalisation. Thus, the parameterised loss itself is not new. In this paper, \gls{gbs} refers to the finite-data framework that motivates and studies the training-time coefficient through the preference issue, model-induced class marginals and class-wise gradients. When $\beta=1$, \eqnref{eq:gbs} also coincides with Balanced Softmax~\citep{Ren2020balms} and with the unit adjustment supported by the Fisher-consistency argument of~\citet{menon2021longtail}. For $\beta\neq1$, that unit-adjustment guarantee does not generally apply, so $\beta$ is treated as a finite-data hyperparameter and selected empirically. The loss has negligible computational overhead and introduces no additional trainable parameters. Although the experiments focus on image classification, the same loss can also be used in other long-tailed classification components, such as those in instance segmentation and object detection~\citep{gupta2019lvis}.

\paragraph{Gradient interpretation of the \texorpdfstring{\gls{gbs}}{GBS}}
During training, \eqnref{eq:gbs} replaces the standard softmax probability $p_{nc}$ in \eqnref{eq:delta_grad} with the adjusted probability $q_{nc}$. For fixed logits, the factor $\card{\set{N}_c}^{\beta}$ increases the conditional probability of a more frequent class relative to that of a less frequent class. Consequently, the adjustment tends to reduce the reward and increase the penalty for head classes, while increasing the reward and reducing the penalty for tail classes. It therefore counteracts the class-count-driven gradient disparity identified in \thmref{thm:ib_grad}. This interpretation does not guarantee balanced gradients throughout optimisation; the empirical gradient behaviour is evaluated in \secref{sec:better_gradients}.

\section{Experiments}~\label{sec:experiments}
This section evaluates the finite-data behaviour of \gls{gbs}. After describing the datasets (\secref{sec:datasets}) and experimental setup (\secref{sec:setup}), we compare the selected values of $\beta$ with existing methods, placing particular emphasis on Balanced Softmax as the fixed $\beta=1$ reference (\secref{sec:compare_bs}). The evaluation considers the preference issue (\secref{sec:less_preference_issue}), average testing recall (\secref{sec:better_recall}), performance on large and highly imbalanced datasets (\secref{sec:scalability}), and class-wise gradient behaviour (\secref{sec:better_gradients}).

\subsection{Datasets}~\label{sec:datasets}

The experiments use the long-tailed CIFAR10 and CIFAR100 benchmarks~\citep{cao2019learning}, abbreviated as C10-LT* and C100-LT* in the tables and figures; ImageNet-1K~\citep{deng2009imagenet}; ImageNet-LT and Places-LT~\citep{openlongtailrecognition2019}; iNaturalist2018~\citep{van2018inaturalist}; and an LVIS-based image-classification dataset. Their characteristics are summarised in \tabref{tab:dataset_summary}.

\begin{table}[!t]
   \centering\small\setlength{\tabcolsep}{3pt}
   \caption{Summary of the datasets used in the experiments.}\label{tab:dataset_summary}

   \begin{tabular}{@{}lrrrrrrrrr@{}}
      \toprule
      \multirow{2}{*}{Dataset} & \multirow{2}{*}{\makecell{Imbalance \\ level}}                & \multicolumn{4}{c}{\#Classes}            & \multicolumn{2}{c}{\#Training samples} & \multicolumn{2}{c}{\#Testing samples} \\\cmidrule{3-6}\cmidrule{7-8}\cmidrule{9-10}
      & & Head    & Medium      & Tail    & Total    & Total   & Per class      & Total  & Per class      \\\midrule
      C10    & 1         & 3    & 3      & 4    & 10    & 50000   & 5000  & 10000  & 1000      \\
      C10-LT10    & 10         & 3    & 3      & 4    & 10    & 20431   & 500-5000  & 10000  & 1000      \\
      C10-LT50   & 50        & 3    & 3      & 4    & 10    & 13996   & 100-5000    & 10000  & 1000      \\
      C10-LT100      & 100          & 3   & 3     & 4   & 10   & 12406   & 50-5000       & 10000  & 1000       \\
      C100   & 1         & 35   & 35     & 30   & 100   & 50000   & 500    & 10000  & 100       \\
      C100-LT10   & 10         & 35   & 35     & 30   & 100   & 19573   & 50-500    & 10000  & 100       \\
      C100-LT50  & 50        & 35   & 35     & 30   & 100   & 12608   & 10-500     & 10000  & 100       \\
      C100-LT100     & 100 & 35  & 35    & 30  & 100  & 10847 & 5-500  & 10000  & 100        \\
      ImageNet-1K     & $\approx$1        & 385  & 479    & 136  & 1000  & 1281167  & 732-1300    & 50000  & 50        \\
      ImageNet-LT        & 256        & 385  & 479    & 136   & 1000   & 115846   & 5-1280    & 50000  & 50       \\
      Places-LT & 996        & 131  & 163   & 71 & 365  & 62500  & 5-4980    & 36500  & 100         \\
      iNaturalist2018       & 500      & 842  & 4076    & 3224  & 8142  & 437513 & 2-1000   & 24426 & 3    \\
      LVIS & 50550 & 405 & 461 & 337 & 1203 & 1269748 &  1-50550 & 244645 & 1-9156 \\
      \bottomrule
   \end{tabular}
\end{table}

CIFAR10 and CIFAR100 contain $32\times32$ images~\citep{krizhevsky2009learning}. Their long-tailed variants are generated using the exponential protocol of~\citet{cao2019learning}. For CIFAR10, classes 0--2, 3--5 and 6--9 form the head, medium and tail groups, respectively; for CIFAR100, the corresponding groups are classes 0--34, 35--69 and 70--99. Each dataset has variants with imbalance levels of 10, 50 and 100. For example, C10-LT10 denotes the CIFAR10 variant with an imbalance level of 10.

ImageNet-1K and ImageNet-LT contain the same classes and testing samples but different training samples~\citep{openlongtailrecognition2019}. Following~\citet{cui2023reslt}, a class is assigned to the head group when it has more than 100 training samples, to the medium group when it has 20--100 samples, and to the tail group when it has fewer than 20 samples. The same thresholds are used for ImageNet-1K, Places-LT and iNaturalist2018.

For the LVIS-based image-classification dataset, objects are cropped from the LVIS-V1 training and testing images using their annotated bounding boxes. The resulting dataset contains 1,203 classes and has training- and testing-set imbalance levels of 50,550 and 9,156, respectively.

\subsection{Experimental setup}~\label{sec:setup}
The experimental configuration follows~\citet{cui2023reslt}. Unless otherwise stated, each model is trained for 200 epochs using \gls{sgd}, which has been shown to generalise favourably compared with adaptive optimisers, e.g., Adam, in image-classification settings~\citep{wilson2017marginal}. The BCL~\citep{zhu2022balanced} and ProCo~\citep{du2024probabilistic} experiments use the 90-epoch schedule of~\citet{du2024probabilistic}.

For the long-tailed CIFAR datasets, we use one NVIDIA A100 GPU, a batch size of 128, SGD with momentum 0.9, and weight decay 0.0005. The initial learning rate is 0.1 and is multiplied by 0.1 at epochs 160 and 180, with linear warm-up during the first five epochs. Training images are randomly cropped to $32\times32$ after four-pixel padding, then randomly flipped horizontally and normalised. Testing images are normalised without augmentation. The model is ResNet32~\citep{he2016deep}.

For ImageNet-1K, ImageNet-LT, Places-LT and iNaturalist2018, we use four NVIDIA A100 GPUs, a total batch size of 256, and SGD with momentum 0.9. The BCL~\citep{zhu2022balanced} and ProCo~\citep{du2024probabilistic} experiments on Places-LT use a total batch size of 128. The learning rate follows cosine annealing from 0.1 to 0~\citep{loshchilov2017sgdr}. The weight-decay values for iNaturalist2018, Places-LT, ImageNet-LT and ImageNet-1K are 0.0001, 0.0005, 0.0005 and 0.0001, respectively. Training uses random resized crops of size $224\times224$, random horizontal flipping and normalisation; ColorJitter is additionally applied on Places-LT, ImageNet-LT and ImageNet-1K. During testing, images are resized to $256\times256$, centre-cropped to $224\times224$, and normalised. We use ResNet50 for ImageNet-1K, ImageNet-LT and iNaturalist2018, and ResNet152 for Places-LT~\citep{he2016deep}.

The setup for the LVIS-based image-classification dataset follows~\citet{he2017mask}. We train an ImageNet-1K-pretrained ResNet50 for 28 epochs. Images are normalised, and no data augmentation is applied during training.

As discussed in \secref{sec:gbs}, $\beta\geq0$ controls the strength of the training-time adjustment. The case $\beta=0$ reduces to standard softmax, whereas $\beta=1$ coincides with Balanced Softmax. Because the preferred adjustment strength can depend on the finite dataset and model, we select $\beta$ empirically by grid search. We first evaluate values from 1.0 to 1.3 in increments of 0.1 and extend the range when the observed trend indicates that larger values should be considered.
\tabref{tab:sensitivity} reports training and testing recall for different values of $\beta$ on Places-LT. Increasing $\beta$ shifts recall from head classes towards tail classes. In this experiment, average testing recall increases up to $\beta=1.3$ and then decreases because the subsequent tail-class gains no longer offset the head-class losses. The selected value $\beta=1.3$ therefore gives the highest average testing recall among the reported candidates.
Detailed configurations regarding $\beta$ for different datasets are available in the accompanying source code.

\begin{table}[!t]\centering\small
\caption{Sensitivity analysis of $\beta$ on Places-LT. Higher values in each column are shown with darker shading.}\label{tab:sensitivity}
\begin{tabular}{lrrrrrrrrrr}\toprule
\multirow{2}{*}{$\beta$} &\multicolumn{4}{c}{Training recall} & &\multicolumn{4}{c}{Testing recall} \\\cmidrule{2-5}\cmidrule{7-10}
&Head &Medium &Tail &All & &Head &Medium &Tail &All \\\midrule
1.1 &\cellcolor[HTML]{888888}84.15 &\cellcolor[HTML]{b8b8b8}84.90 &\cellcolor[HTML]{f5f5f5}82.74 &\cellcolor[HTML]{b8b8b8}84.21 & &\cellcolor[HTML]{a8a8a8}35.46 &\cellcolor[HTML]{f5f5f5}28.86 &\cellcolor[HTML]{f5f5f5}20.39 &\cellcolor[HTML]{f5f5f5}29.58 \\
1.2 &\cellcolor[HTML]{a8a8a8}83.30 &\cellcolor[HTML]{a8a8a8}85.42 &\cellcolor[HTML]{dcdcdc}83.78 &\cellcolor[HTML]{888888}84.34 & &\cellcolor[HTML]{888888}35.46 &\cellcolor[HTML]{a8a8a8}29.69 &\cellcolor[HTML]{dcdcdc}23.83 &\cellcolor[HTML]{a8a8a8}30.62 \\
1.3 &\cellcolor[HTML]{b8b8b8}81.95 &\cellcolor[HTML]{888888}85.79 &\cellcolor[HTML]{c8c8c8}84.96 &\cellcolor[HTML]{a8a8a8}84.25 & &\cellcolor[HTML]{b8b8b8}34.85 &\cellcolor[HTML]{888888}29.90 &\cellcolor[HTML]{c8c8c8}24.65 &\cellcolor[HTML]{888888}30.65 \\
1.4 &\cellcolor[HTML]{c8c8c8}79.61 &\cellcolor[HTML]{dcdcdc}84.00 &\cellcolor[HTML]{888888}87.95 &\cellcolor[HTML]{c8c8c8}83.19 & &\cellcolor[HTML]{c8c8c8}33.05 &\cellcolor[HTML]{c8c8c8}29.58 &\cellcolor[HTML]{b8b8b8}26.06 &\cellcolor[HTML]{c8c8c8}30.14 \\
1.5 &\cellcolor[HTML]{dcdcdc}77.98 &\cellcolor[HTML]{c8c8c8}84.02 &\cellcolor[HTML]{b8b8b8}86.48 &\cellcolor[HTML]{dcdcdc}82.33 & &\cellcolor[HTML]{dcdcdc}31.99 &\cellcolor[HTML]{dcdcdc}29.33 &\cellcolor[HTML]{a8a8a8}27.75 &\cellcolor[HTML]{dcdcdc}29.98 \\
1.6 &\cellcolor[HTML]{f5f5f5}76.87 &\cellcolor[HTML]{f5f5f5}83.82 &\cellcolor[HTML]{a8a8a8}87.30 &\cellcolor[HTML]{f5f5f5}82.00 & &\cellcolor[HTML]{f5f5f5}31.17 &\cellcolor[HTML]{b8b8b8}29.68 &\cellcolor[HTML]{888888}29.35 &\cellcolor[HTML]{b8b8b8}30.15 \\
\bottomrule
\end{tabular}
\end{table}

\subsection{Comparison with existing methods}\label{sec:compare_bs}
We compare \gls{gbs} with representative rebalancing methods: Balanced Softmax~\citep{Ren2020balms}, Focal Loss~\citep{lin2017focal}, $\tau$-norm and \gls{crt}~\citep{Kang2020Decoupling}, post-hoc logit adjustment~\citep{menon2021longtail}, \gls{reslt}~\citep{cui2023reslt}, and \gls{ldam}~\citep{cao2019learning}. Balanced Softmax is the principal reference because it is the widely adopted $\beta=1$ member of the same training-time loss family and is used by recent representation-learning methods~\citep{zhu2022balanced,du2024probabilistic}. We additionally report the non-softmax-based CAL~\citep{yanneural} and DisA~\citep{gao2024distribution} methods. To examine compatibility with representation learning, we combine \gls{gbs} with BCL~\citep{zhu2022balanced} and ProCo~\citep{du2024probabilistic}.

All results were obtained using the experimental setup described above, except those for CAL and DisA, which are taken from their original publications because implementations for the relevant datasets were unavailable. Those reported results use the same backbone and evaluation protocol considered here.


\subsubsection{Mitigating the preference issue}\label{sec:less_preference_issue}

We evaluate the preference issue using the model imbalance level $I$ defined in \secref{sec:ib_train_data_impact}. For \gls{gbs} and Balanced Softmax, \figref{fig:model_imbalance_level} reports the mean $I$ over five independent runs on each dataset. The standard-softmax baseline is based on one run because of computational constraints.
Standard softmax has the highest value of $I$ on every evaluated dataset, with particularly large disparities on ImageNet-LT and the LVIS-based dataset. This observation is consistent with the class-count-driven gradient disparity analysed in \secref{sec:ib_train_data_impact}. Balanced Softmax~\citep{Ren2020balms} reduces $I$, while the selected values of $\beta$ (\gls{gbs}) reduce it further on all evaluated datasets. The results on iNaturalist2018 and the LVIS-based dataset are examined further in \secref{sec:scalability}.

\begin{figure}[!tp]
    \centering
    \includegraphics[width=\textwidth]{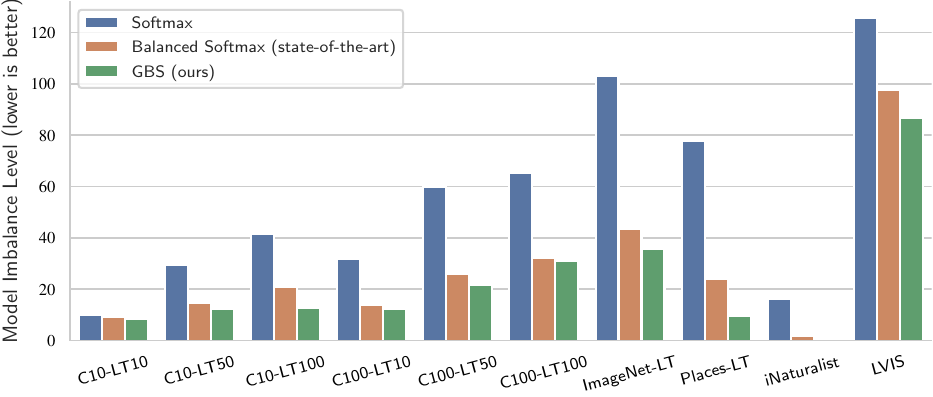}
    \caption{Model imbalance level $I$ on the evaluated long-tailed datasets. Lower values indicate less disparity among the head-, medium- and tail-group preference scores. The selected values of $\beta$ used by \gls{gbs} produce lower $I$ than standard softmax and the $\beta=1$ Balanced Softmax reference.}\label{fig:model_imbalance_level}
\end{figure}

\subsubsection{Average testing recall}\label{sec:better_recall}
We next evaluate the average testing recall across all classes, a standard metric in long-tailed recognition~\citep{Kang2020Decoupling,cui2023reslt}. \tabsref{tab:test_recall_cifars}, \ref{tab:test_recall_imagenet_lt}, and~\ref{tab:test_recall_place_lt} also report recall separately for the head, medium and tail groups to show how each method distributes performance across the class-frequency range.
For \gls{gbs} and Balanced Softmax, the tables report the $\mathrm{mean}\pm\mathrm{standard\ deviation}$ over five independent runs. \figref{fig:bs_gbs_img_place_box} visualises these runs for ImageNet-LT and Places-LT; the corresponding plots for the long-tailed CIFAR datasets are provided in \appref{ssec:cifar_statistic}.
\begin{figure}[!tp]
    \centering
    \includegraphics[width=\linewidth]{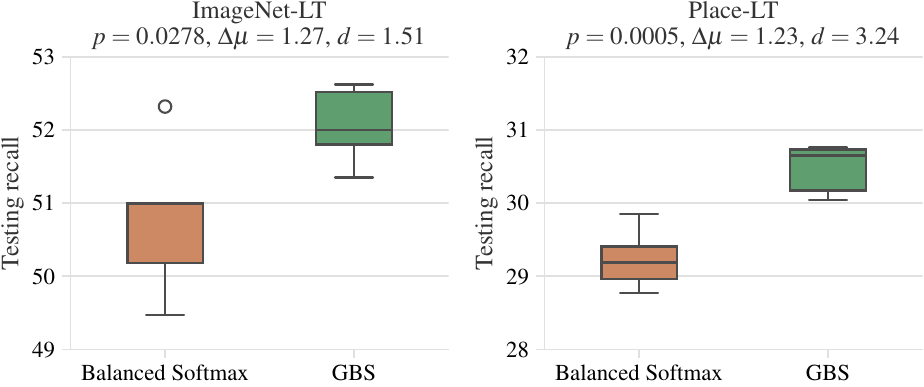}
    \caption{Average testing recall for Balanced Softmax and \gls{gbs} on ImageNet-LT and Places-LT over five independent runs. The annotations report Welch's $t$-test $p$-value, the mean difference $\Delta\mu$, and Cohen's $d$ effect size.}
    \label{fig:bs_gbs_img_place_box}
\end{figure}
\begin{table}[!t]\centering\small
\caption{Average testing recall for \gls{gbs} and Balanced Softmax on long-tailed CIFAR datasets. Results are the mean $\pm$ standard deviation over five independent runs. Bold values indicate the higher average recall across all classes for each dataset.}\label{tab:test_recall_cifars}
\setlength{\tabcolsep}{4pt}
\begin{tabular}{@{}llrrrr@{}}\toprule
Method &Dataset &Head &Medium &Tail  &All \\\midrule
\multirow{6}{*}{\makecell[l]{Balanced\\Softmax}}
&C10-LT10   &93.35$\pm$0.42 & 84.20$\pm$0.60 & 89.02$\pm$0.47 & 88.87$\pm$0.32 \\
&C10-LT50   &92.68$\pm$0.75 & 78.23$\pm$0.42 & 78.22$\pm$0.97 & 82.56$\pm$0.47 \\
&C10-LT100  &92.40$\pm$0.26 & 75.77$\pm$1.20 & 71.43$\pm$1.38 & 79.02$\pm$0.59 \\     
&C100-LT10  &66.80$\pm$0.76 & 59.83$\pm$0.89 & 53.39$\pm$0.70 & 60.34$\pm$0.35 \\   
&C100-LT50  &63.25$\pm$0.92 & 48.59$\pm$1.18 & 37.03$\pm$1.28 & 50.25$\pm$0.61 \\
&C100-LT100 &61.31$\pm$0.98 & 44.29$\pm$0.80 & 28.74$\pm$0.52 & 45.58$\pm$0.34 \\\midrule
\multirow{6}{*}{\gls{gbs}}
&C10-LT10   &92.37$\pm$0.37 & 84.07$\pm$0.49 & 90.87$\pm$0.31 & \textbf{89.28$\pm$0.22}\\ 
&C10-LT50   &90.67$\pm$0.27 & 78.45$\pm$0.60 & 83.01$\pm$0.98 & \textbf{83.94$\pm$0.27}\\ 
&C10-LT100  &87.45$\pm$1.08 & 74.79$\pm$1.51 & 81.90$\pm$0.84 & \textbf{81.43$\pm$0.24}\\  
&C100-LT10  &66.50$\pm$1.12 & 59.38$\pm$1.02 & 54.36$\pm$1.24 & \textbf{60.37$\pm$0.61}\\   
&C100-LT50  &62.18$\pm$0.22 & 48.33$\pm$1.37 & 40.09$\pm$1.57 & \textbf{50.71$\pm$0.90}\\
&C100-LT100 &61.57$\pm$0.55 & 44.39$\pm$0.63 & 29.65$\pm$1.52 & \textbf{45.98$\pm$0.52}\\ 
\bottomrule
\end{tabular}
\end{table}
\begin{table}[!t]\centering\small
\caption{Testing recall on ImageNet-LT.
For Balanced Softmax and \gls{gbs}, results represent the mean $\pm$ standard deviation across five independent runs.
Bold values indicate the higher average recall within each direct comparison.}\label{tab:test_recall_imagenet_lt}
\setlength{\tabcolsep}{4pt}
\begin{tabularx}{\textwidth}{@{}Xrrrr@{}}\toprule
\multirow{2}{*}{Method} &\multicolumn{4}{c}{Testing recall} \\\cmidrule{2-5}
&Head &Medium &Tail &All \\\midrule
\multicolumn{5}{c}{ResNet50 + a softmax-based classifier} \\ \midrule
Standard softmax~\citep{goodfellow2016deep} &66.03 &38.13 &8.76 &44.88 \\
Focal loss~\citep{lin2017focal} &64.89 &36.53 &8.19 &43.60 \\
$\tau$-norm~\citep{Kang2020Decoupling} &60.15 &46.82 &31.35 &49.85 \\
Post-hoc logit adjustment~\citep{menon2021longtail} &61.43 &47.56 &28.18 &50.27 \\
\gls{reslt}~\citep{cui2023reslt} &54.16 &50.85 &40.26 &50.68 \\
\gls{crt}~\citep{Kang2020Decoupling} &61.86 &45.89 &26.97 &49.46 \\
\gls{ldam}~\citep{cao2019learning} &64.36 &46.89 &25.57 &50.72 \\
Balanced Softmax~\citep{Ren2020balms} &62.09$\pm$0.84 & 48.36$\pm$0.82 & 29.36$\pm$0.89 & 51.06$\pm$0.82 \\
\gls{gbs} &61.13$\pm$0.70 & 49.16$\pm$0.69 & 32.91$\pm$0.67 & \textbf{51.56$\pm$0.65} \\\midrule
\multicolumn{5}{c}{ResNet50 + a non-softmax-based classifier} \\\midrule
CAL~\citep{yanneural} &- &- &- &49.70 \\
DisA~\citep{gao2024distribution} &67.70 &38.60 &7.30 &44.80 \\\midrule
\multicolumn{5}{c}{ResNet50 + Contrastive learning + softmax-based classifier} \\\midrule
BCL+Balanced Softmax~\citep{zhu2022balanced} &66.04 &53.85 &36.26 &56.15 \\
BCL+\gls{gbs} &{64.58} &{54.35} &{39.90} &\textbf{56.32} \\
ProCo+Balanced Softmax~\citep{du2024probabilistic} &66.44 &54.75 &37.16 &56.86 \\
ProCo+\gls{gbs} &{65.28} &{55.38} &{41.28} &\textbf{57.27} \\
\bottomrule
\end{tabularx}
\end{table}
\begin{table}[!tp]
   \centering\small
   \caption{Testing recall on Places-LT.
   For Balanced Softmax and \gls{gbs}, results represent the mean $\pm$ standard deviation across five independent runs.
Bold values indicate the higher average recall within each direct comparison.}\label{tab:test_recall_place_lt}
\setlength{\tabcolsep}{4pt}
\begin{tabularx}{\textwidth}{@{}Xrrrr@{}}\toprule
\multirow{2}{*}{Method} &\multicolumn{4}{c}{Testing recall} \\\cmidrule{2-5}
&Head &Medium &Tail &All \\\midrule
\multicolumn{5}{c}{ResNet152 + classifier} \\\midrule
Standard softmax &38.52 &18.28 &3.51 &22.67 \\
Focal loss &39.36 &18.20 &4.66 &23.16 \\
$\tau$-norm &34.85 &29.16 &21.48 &29.71 \\
Post-hoc logit adjustment &36.11 &28.53 &17.86 &29.18 \\
\gls{reslt} &34.72 &28.93 &19.10 &29.10 \\
\gls{crt} &37.32 &25.76 &14.41 &27.70 \\
\gls{ldam} &36.37 &16.06 &5.08 &21.22 \\
Balanced Softmax &36.65$\pm$0.57 & 28.35$\pm$0.47 & 18.52$\pm$0.66 & 29.42$\pm$0.45 \\
\gls{gbs} &34.23$\pm$0.64 & 29.37$\pm$0.50 & 24.20$\pm$0.48 & \textbf{30.11$\pm$0.45} \\\midrule
\multicolumn{5}{c}{ResNet152 + Contrastive learning + classifier} \\\midrule
BCL+Balanced Softmax~\citep{zhu2022balanced} &39.89 &35.23 &22.92 &34.50 \\
BCL+\gls{gbs} &{37.59} &{36.66} &{27.11} &\textbf{35.13} \\
ProCo+Balanced Softmax~\citep{du2024probabilistic} &40.44 &34.71 &23.21 &34.53 \\
ProCo+\gls{gbs} &{37.76} &{36.05} &{27.06} &\textbf{34.92} \\
\bottomrule
\end{tabularx}
\end{table}

Across the datasets in \tabsref{tab:test_recall_cifars}, \ref{tab:test_recall_imagenet_lt}, and~\ref{tab:test_recall_place_lt}, the selected values of $\beta$ yield higher average testing recall than the $\beta=1$ Balanced Softmax reference. The gains generally arise from improved medium- and tail-class recall, accompanied by some reduction in head-class recall. The same pattern is observed when the two losses are combined with BCL and ProCo, showing that the adjustable training-time coefficient remains useful with stronger representations.

For ImageNet-LT and Places-LT, the five-run comparisons in \figref{fig:bs_gbs_img_place_box} give $p<0.05$ under Welch's $t$-test and Cohen's $d>0.8$. The corresponding long-tailed CIFAR statistics are reported in \appref{ssec:cifar_statistic}.

\subsubsection{Large-scale and severe-imbalance settings}~\label{sec:scalability}
We further evaluate \gls{gbs} when the number of classes is large or the class-frequency imbalance is severe. iNaturalist2018 contains 8,142 classes, while the dataset constructed from LVIS-V1~\citep{gupta2019lvis} has a training-set imbalance level of 50,550 (\secref{sec:datasets}). Because established configurations for the other baselines are unavailable for the LVIS-based classification dataset, this evaluation focuses on the direct comparison between the selected $\beta$ (\gls{gbs}) and the $\beta=1$ Balanced Softmax reference.
\tabsref{tab:test_recall_int18} and~\ref{tab:test_recall_lvis} report the $\mathrm{mean}\pm\mathrm{standard\ deviation}$ over five independent runs, and \figref{fig:bs_gbs_int_lvis_box} visualises the corresponding distributions.

\begin{figure}[!tp]
    \centering
    \includegraphics[width=\linewidth]{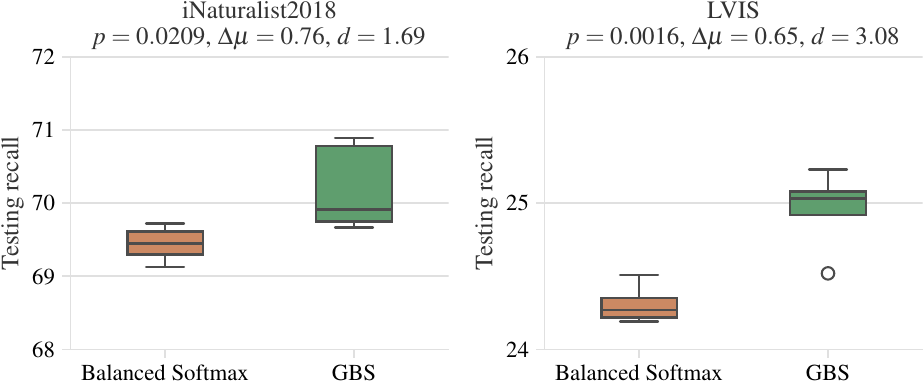}
    \caption{Average testing recall for Balanced Softmax and \gls{gbs} on iNaturalist2018 and the LVIS-based dataset over five independent runs. The annotations report Welch's $t$-test $p$-value, the mean difference $\Delta\mu$, and Cohen's $d$ effect size.}
    \label{fig:bs_gbs_int_lvis_box}
\end{figure}

\begin{table}[!t]
   \centering\small
   \caption{Testing recall on iNaturalist2018. For Balanced Softmax and \gls{gbs}, results are the mean $\pm$ standard deviation over five independent runs. The bold value indicates the highest average recall across all classes.}\label{tab:test_recall_int18}
   \begin{tabular}{lrrrrr}
      \toprule
      \multirow{2}{*}{Method} & \multicolumn{4}{c}{Testing recall} \\\cmidrule{2-5}
      &Head &Medium &Tail &All     \\\midrule
      Standard softmax & 75.26                         & 66.24                         & 60.03 & 64.71                         \\
      Focal loss        & 73.56                         & 64.65                         & 57.48                         & 62.73                         \\
      $\tau$-norm        & 70.27                         & 68.52                         & 68.74                         & 68.79                         \\
      Post-hoc logit adjustment & 69.28                         & 69.27                         & 69.83                         & 69.49                         \\
      \gls{reslt} & 68.49                         & 66.23 & 74.21                         & 69.62                         \\
      \gls{crt}   & 72.88                         & 69.05                         & 66.24                         & 68.34                         \\
      \gls{ldam}   & 64.09                         & 61.47                         & 59.07                         & 60.79 \\
      Balanced Softmax   & 69.26$\pm$0.28 &  69.33$\pm$0.34 & 69.63$\pm$0.39 & \text{69.44$\pm$0.24} \\
      \gls{gbs} & 70.67$\pm$0.69 & 70.21$\pm$0.70 & 70.06$\pm$0.44 & \textbf{70.20$\pm$0.59} \\
      \bottomrule
   \end{tabular}
\end{table}

\begin{table}[!t]\centering\small
\caption{Testing recall on the LVIS-based image-classification dataset. Results for Balanced Softmax and \gls{gbs} are the mean $\pm$ standard deviation over five independent runs. The bold value indicates the higher average recall across all classes.}\label{tab:test_recall_lvis}
\begin{tabular}{lrrrrr}\toprule
Method &Head &Medium &Tail &All \\\midrule
Standard softmax &45.54 &8.99 &0.33 &18.87 \\
Balanced Softmax & 42.97$\pm$0.18 & 23.92$\pm$0.18 & 2.42$\pm$0.30 & {24.31$\pm$0.13}  \\
\gls{gbs} &40.85$\pm$0.13 & 26.81$\pm$0.65 & 3.33$\pm$0.11 &  \textbf{24.96$\pm$0.27}  \\
\bottomrule
\end{tabular}
\end{table}

In \tabsref{tab:test_recall_int18} and~\ref{tab:test_recall_lvis}, the selected values of $\beta$ achieve higher average testing recall than Balanced Softmax. The five-run comparisons in \figref{fig:bs_gbs_int_lvis_box} give $p<0.05$ under Welch's $t$-test and Cohen's $d>0.8$. These results show that the finite-data benefit over the fixed $\beta=1$ reference is also observed with thousands of classes and with an extremely sparse tail.

\subsubsection{Class-wise gradient behaviour}~\label{sec:better_gradients}
We finally examine whether the empirical gradient behaviour is consistent with the interpretation in \secref{sec:gbs}. Recall that \eqnref{eq:delta_grad} separates the class-wise update signal into a reward from samples belonging to class $c$ and a penalty from samples belonging to the other classes.

For each class group, \figref{fig:reward_penalty_gbs} reports the average relative difference
\begin{equation}
   \frac{\mathtt{reward}-\mathtt{penalty}}{\mathtt{reward}}\times100.
\end{equation}
The quantities were accumulated over one epoch without parameter updates, thereby isolating the update signals generated by each objective at a fixed model state. A higher value indicates a larger net reward relative to the penalty.

On the balanced CIFAR10, CIFAR100 and ImageNet-1K datasets, the group-wise values are relatively similar. On the imbalanced datasets, standard softmax produces larger values for head classes than for tail classes, consistent with the class-count effect in \thmref{thm:ib_grad}. The selected \gls{gbs} adjustment increases the relative update signal for tail classes and reduces the head-tail disparity. In some settings the tail-group value exceeds the head-group value, indicating that the selected $\beta$ compensates for, rather than exactly equalises, the class-wise signals.

\begin{figure}[!tp]
    \centering
    \includegraphics[width=\textwidth]{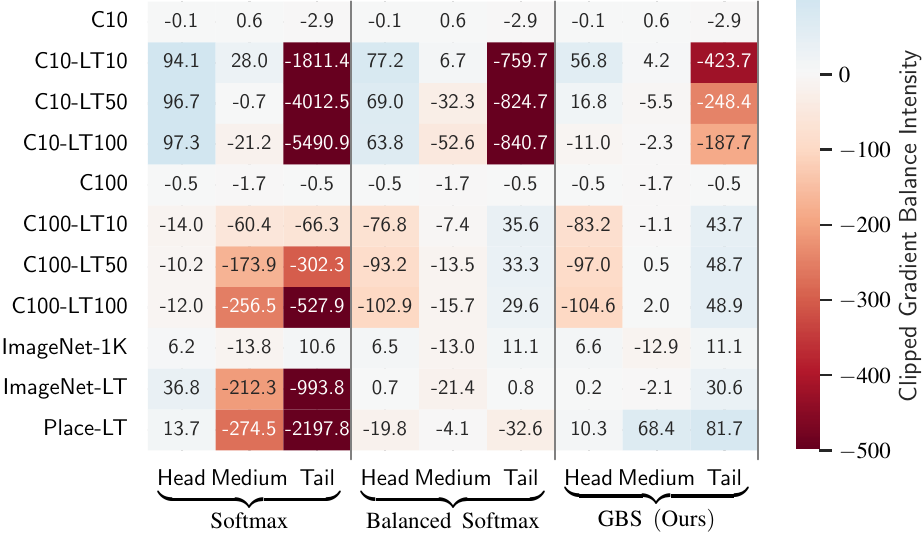}
    \caption{Average relative difference between the class-wise reward and penalty signals for the head, medium and tail groups. Higher values indicate a larger net reward relative to the penalty. On imbalanced datasets, \gls{gbs} shifts the relative update signal towards tail classes compared with standard softmax. Values below -500 are displayed at -500 to keep the vertical scale readable, while the numerical labels show the actual, unclipped values.}\label{fig:reward_penalty_gbs}
\end{figure}

\section{Conclusion}\label{sec:conclusion}
This paper has presented a finite-data perspective on training-time logit adjustment for long-tailed recognition. The parameterised loss used by the \gls{gbs} framework is algebraically identical to the training-time logit-adjusted loss of \citet{menon2021longtail}, with $\tau=\beta$, and its $\beta=1$ case coincides with Balanced Softmax. Building on this loss family, the paper contributes a finite-data framework for examining the choice of $\beta$. Within this framework, we characterised the preference issue as the class-wise disparity in training error and generalisation gap, introduced the model imbalance level $I$ to quantify it, and analysed how imbalanced training data affects class-wise gradients under standard softmax. We also examined the relationship between class frequency and model-induced class marginals to motivate the adjustable coefficient.

Across the evaluated long-tailed benchmarks, $\beta=1$ did not attain the highest average testing recall on most datasets. Relative to the $\beta=1$ Balanced Softmax reference, the selected values of $\beta$ produced a lower model imbalance level and higher average testing recall. The results also showed that the adjustment remains compatible with the BCL and ProCo representation-learning frameworks. These findings complement the Fisher-consistency justification for the unit adjustment under the true data distribution by showing that a different coefficient can be preferable when training on finite data.

Despite these strengths, this work highlights several limitations that provide opportunities for future research:
\begin{itemize}
    \item The power-law approximations relating class frequency to the model-induced class marginals are supported empirically only in the studied settings and are not established as a universal property of long-tailed recognition. They describe the observed trends using class frequency alone, although data diversity, class-specific noise and intra-class variance may also influence these marginals. Developing more sophisticated approximations that incorporate these factors is a promising direction.
    \item The performance of \gls{gbs} remains dependent on the value of $\beta$, which is selected empirically by grid search. Future research could investigate methods for adapting or learning this coefficient during training.
    \item Like other rebalancing methods, the \gls{gbs} also suffers from the head-tail trade-off, where gains in tail-class recall often come at the cost of head-class accuracy. To overcome this, integrating the \gls{gbs} with advanced representation learning techniques, such as contrastive learning, may offer a way to improve tail performance without sacrificing head-class recall.
    \item Finally, this study focused on image classification. Evaluating the framework in other long-tailed classification components, such as those used in instance segmentation and object detection, remains a direction for future work.
\end{itemize}

\section*{Code Availability}
The code is available at \href{https://github.com/YihangZhu/Class-Balanced-Softmax.git}{https://github.com/YihangZhu/Class-Balanced-Softmax.git}.

\section*{Acknowledgments}
Both Yi-Hang Zhu and Xinyu Yang are supported by College of Science and Engineering Scholarships from the University of Leicester. 
Shiqi Su is supported by a PhD studentship from the Science and Technology Facilities Council (STFC) and a College of Science and Engineering Scholarship from the University of Leicester.
We thank Sulis (HPC facility) for providing computational
resources to support this research.

\section*{Declaration of generative AI and AI-assisted technologies in the manuscript preparation process}
During the preparation of this work, the authors used Gemini (a large language model built by Google) in order to improve the language, readability, and grammatical accuracy of the manuscript. After using this tool, the authors reviewed and edited the content as needed and take full responsibility for the content of the manuscript.

\bibliographystyle{apalike}
\phantomsection
\addcontentsline{toc}{section}{References}
\bibliography{gbs_arXiv}

@article{zhu2026cbs,
  title={{Class-Balanced Softmax}: A {Bayes} theory-based method for long-tailed recognition},
  author={Zhu, Yi-Hang and Raman, Rajeev and Su, Shiqi and Sun, Jianyuan and Yang, Xinyu and Xing, Nan and Zhou, Huiyu},
  journal={Pattern Recognition},
  eid={114961},
  year={2026},
  doi={10.1016/j.patcog.2026.114961},
  url={https://doi.org/10.1016/j.patcog.2026.114961},
  note={Article 114961. Available online 19 September 2026; journal pre-proof. \url{https://doi.org/10.1016/j.patcog.2026.114961}}
}

@article{pan2021model,
  title={On model calibration for long-tailed object detection and instance segmentation},
  author={Pan, Tai-Yu and Zhang, Cheng and Li, Yandong and Hu, Hexiang and Xuan, Dong and Changpinyo, Soravit and Gong, Boqing and Chao, Wei-Lun},
  journal={Proceedings of Neural Information Processing Systems (NeurIPS)},
  volume={34},
  pages={2529--2542},
  year={2021}
}

@article{zhang2025class,
  title={Class-aware {Universum} Inspired re-balance learning for long-tailed recognition},
  author={Zhang, Enhao and Geng, Chuanxing and Chen, Songcan},
  journal={Pattern Recognition},
  volume={161},
  pages={111337},
  year={2025},
  publisher={Elsevier}
}

@article{tan2024ncl,
  title={{NCL++}: Nested collaborative learning for long-tailed visual recognition},
  author={Tan, Zichang and Li, Jun and Du, Jinhao and Wan, Jun and Lei, Zhen and Guo, Guodong},
  journal={Pattern Recognition},
  volume={147},
  pages={110064},
  year={2024},
  publisher={Elsevier}
}

@article{xiang2023margin,
  title={Margin-aware rectified augmentation for long-tailed recognition},
  author={Xiang, Liuyu and Han, Jungong and Ding, Guiguang},
  journal={Pattern Recognition},
  volume={141},
  pages={109608},
  year={2023},
  publisher={Elsevier}
}

@article{baik2024dbn,
  title={{DBN-Mix}: Training dual branch network using bilateral mixup augmentation for long-tailed visual recognition},
  author={Baik, Jae Soon and Yoon, In Young and Choi, Jun Won},
  journal={Pattern Recognition},
  volume={147},
  pages={110107},
  year={2024},
  publisher={Elsevier}
}

@article{liu2024lcreg,
  title={{LCReg}: Long-tailed image classification with latent categories based recognition},
  author={Liu, Weide and Wu, Zhonghua and Wang, Yiming and Ding, Henghui and Liu, Fayao and Lin, Jie and Lin, Guosheng},
  journal={Pattern Recognition},
  volume={145},
  pages={109971},
  year={2024},
  publisher={Elsevier}
}

@inproceedings{zhu2022balanced,
	title        = {Balanced contrastive learning for long-tailed visual recognition},
	author       = {Zhu, Jianggang and Wang, Zheng and Chen, Jingjing and Chen, Yi-Ping Phoebe and Jiang, Yu-Gang},
	year         = 2022,
	booktitle    = {Proceedings of the IEEE/CVF Conference on Computer Vision and Pattern Recognition (CVPR)},
	pages        = {6908--6917}
}

@article{du2024probabilistic,
	title        = {Probabilistic contrastive learning for long-tailed visual recognition},
	author       = {Du, Chaoqun and Wang, Yulin and Song, Shiji and Huang, Gao},
	year         = 2024,
	journal      = {IEEE Transactions on Pattern Analysis and Machine Intelligence},
	publisher    = {IEEE}
}

@inproceedings{gao2024distribution,
	title        = {Distribution alignment optimization through neural collapse for long-tailed classification},
	author       = {Gao, Jintong and Zhao, He and dan Guo, Dan and Zha, Hongyuan},
	year         = 2024,
	booktitle    = {International Conference on Machine Learning (ICML)}
}

@inproceedings{yanneural,
	title        = {Neural Collapse To Multiple Centers For Imbalanced Data},
	author       = {Yan, Hongren and Qian, Yuhua and Peng, Furong and Luo, Jiachen and Li, Feijiang and others},
	year         = 2024,
	booktitle    = {Proceedings of Neural Information Processing Systems (NeurIPS)}
}

@inproceedings{wang2023yolov7,
	title        = {{YOLOv7}: Trainable bag-of-freebies sets new state-of-the-art for real-time object detectors},
	author       = {Wang, Chien-Yao and Bochkovskiy, Alexey and Liao, Hong-Yuan Mark},
	year         = 2023,
	booktitle    = {Proceedings of the IEEE/CVF Conference on Computer Vision and Pattern Recognition (CVPR)},
	pages        = {7464--7475}
}

@inproceedings{gupta2019lvis,
	title        = {{LVIS}: A dataset for large vocabulary instance segmentation},
	author       = {Gupta, Agrim and Dollar, Piotr and Girshick, Ross},
	year         = 2019,
	booktitle    = {Proceedings of the IEEE/CVF Conference on Computer Vision and Pattern Recognition (CVPR)},
	pages        = {5356--5364}
}

@article{yang2022survey,
	title        = {A survey on long-tailed visual recognition},
	author       = {Yang, Lu and Jiang, He and Song, Qing and Guo, Jun},
	year         = 2022,
	journal      = {International Journal of Computer Vision},
	publisher    = {Springer},
	volume       = 130,
	number       = 7,
	pages        = {1837--1872}
}

@inproceedings{he2017mask,
	title        = {{Mask R-CNN}},
	author       = {He, Kaiming and Gkioxari, Georgia and Doll{\'a}r, Piotr and Girshick, Ross},
	year         = 2017,
	booktitle    = {Proceedings of the IEEE International Conference on Computer Vision (ICCV)},
	pages        = {2961--2969}
}

@article{bengio2013representation,
	title        = {Representation learning: A review and new perspectives},
	author       = {Bengio, Yoshua and Courville, Aaron and Vincent, Pascal},
	year         = 2013,
	journal      = {IEEE Transactions on Pattern Analysis and Machine Intelligence},
	publisher    = {IEEE},
	volume       = 35,
	number       = 8,
	pages        = {1798--1828}
}

@inproceedings{krizhevsky2012imagenet,
	title        = {{ImageNet} classification with deep convolutional neural networks},
	author       = {Krizhevsky, Alex and Sutskever, Ilya and Hinton, Geoffrey E},
	year         = 2012,
	booktitle      = {Proceedings of Neural Information Processing Systems (NeurIPS)},
	volume       = 25
}

@inproceedings{he2016deep,
	title        = {Deep residual learning for image recognition},
	author       = {He, Kaiming and Zhang, Xiangyu and Ren, Shaoqing and Sun, Jian},
	year         = 2016,
	booktitle    = {Proceedings of the IEEE/CVF Conference on Computer Vision and Pattern Recognition (CVPR)},
	pages        = {770--778}
}

@inproceedings{dosovitskiy2021an,
	title        = {An Image is Worth 16x16 Words: Transformers for Image Recognition at Scale},
	author       = {Alexey Dosovitskiy and Lucas Beyer and Alexander Kolesnikov and Dirk Weissenborn and Xiaohua Zhai and Thomas Unterthiner and Mostafa Dehghani and Matthias Minderer and Georg Heigold and Sylvain Gelly and Jakob Uszkoreit and Neil Houlsby},
	year         = 2021,
	booktitle    = {International Conference on Learning Representations (ICLR)}
}

@inproceedings{wilson2017marginal,
	title        = {The marginal value of adaptive gradient methods in machine learning},
	author       = {Wilson, Ashia C and Roelofs, Rebecca and Stern, Mitchell and Srebro, Nati and Recht, Benjamin},
	year         = 2017,
	booktitle      = {Proceedings of Neural Information Processing Systems (NeurIPS)},
	volume       = 30
}

@inproceedings{loshchilov2017sgdr,
	title        = {{SGDR}: Stochastic Gradient Descent with Warm Restarts},
	author       = {Ilya Loshchilov and Frank Hutter},
	year         = 2017,
	booktitle    = {International Conference on Learning Representations (ICLR)}
}

@inproceedings{zhang2018mixup,
	title        = {mixup: Beyond Empirical Risk Minimization},
	author       = {Hongyi Zhang and Moustapha Cisse and Yann N. Dauphin and David Lopez-Paz},
	year         = 2018,
	booktitle    = {International Conference on Learning Representations (ICLR)}
}

@inproceedings{deng2009imagenet,
	title        = {{ImageNet}: A large-scale hierarchical image database},
	author       = {Deng, Jia and Dong, Wei and Socher, Richard and Li, Li-Jia and Li, Kai and Fei-Fei, Li},
	year         = 2009,
	booktitle    = {Proceedings of the IEEE/CVF Conference on Computer Vision and Pattern Recognition (CVPR)},
	pages        = {248--255},
	organization = {IEEE}
}

@techreport{krizhevsky2009learning,
	title        = {Learning multiple layers of features from tiny images},
	author       = {Krizhevsky, Alex and Hinton, Geoffrey and others},
	year         = 2009,
	institution  = {University of Toronto}
}

@inproceedings{van2018inaturalist,
	title        = {The {iNaturalist} species classification and detection dataset},
	author       = {Van Horn, Grant and Mac Aodha, Oisin and Song, Yang and Cui, Yin and Sun, Chen and Shepard, Alex and Adam, Hartwig and Perona, Pietro and Belongie, Serge},
	year         = 2018,
	booktitle    = {Proceedings of the IEEE/CVF Conference on Computer Vision and Pattern Recognition (CVPR)},
	pages        = {8769--8778}
}

@article{zhang2023deep,
	title        = {Deep long-tailed learning: A survey},
	author       = {Zhang, Yifan and Kang, Bingyi and Hooi, Bryan and Yan, Shuicheng and Feng, Jiashi},
	year         = 2023,
	journal      = {IEEE Transactions on Pattern Analysis and Machine Intelligence},
	publisher    = {IEEE}
}

@book{goodfellow2016deep,
	title        = {Deep learning},
	author       = {Goodfellow, Ian and Bengio, Yoshua and Courville, Aaron},
	year         = 2016,
	publisher    = {MIT Press}
}

@inproceedings{Ren2020balms,
	title        = {Balanced Meta-Softmax for Long-Tailed Visual Recognition},
	author       = {Jiawei Ren and Cunjun Yu and Shunan Sheng and Xiao Ma and Haiyu Zhao and Shuai Yi and Hongsheng Li},
	year         = 2020,
	month        = {Dec},
	booktitle    = {Proceedings of Neural Information Processing Systems (NeurIPS)}
}

@inproceedings{Kang2020Decoupling,
	title        = {Decoupling Representation and Classifier for Long-Tailed Recognition},
	author       = {Bingyi Kang and Saining Xie and Marcus Rohrbach and Zhicheng Yan and Albert Gordo and Jiashi Feng and Yannis Kalantidis},
	year         = 2020,
	booktitle    = {International Conference on Learning Representations (ICLR)}
}

@article{cui2023reslt,
	title        = {{ResLT}: Residual learning for long-tailed recognition},
	author       = {Cui, Jiequan and Liu, Shu and Tian, Zhuotao and Zhong, Zhisheng and Jia, Jiaya},
	year         = 2023,
	journal      = {IEEE Transactions on Pattern Analysis and Machine Intelligence},
	publisher    = {IEEE},
	volume       = 45,
	number       = 3,
	pages        = {3695--3706}
}

@inproceedings{cao2019learning,
	title        = {Learning Imbalanced Datasets with Label-Distribution-Aware Margin Loss},
	author       = {Cao, Kaidi and Wei, Colin and Gaidon, Adrien and Arechiga, Nikos and Ma, Tengyu},
	year         = 2019,
	booktitle    = {Proceedings of Neural Information Processing Systems (NeurIPS)}
}

@inproceedings{openlongtailrecognition2019,
	title        = {Large-Scale Long-Tailed Recognition in an Open World},
	author       = {Liu, Ziwei and Miao, Zhongqi and Zhan, Xiaohang and Wang, Jiayun and Gong, Boqing and Yu, Stella X.},
	year         = 2019,
	booktitle    = {IEEE Conference on Computer Vision and Pattern Recognition (CVPR)}
}

@article{zhou2005training,
	title        = {Training cost-sensitive neural networks with methods addressing the class imbalance problem},
	author       = {Zhou, Zhi-Hua and Liu, Xu-Ying},
	year         = 2005,
	journal      = {IEEE Transactions on Knowledge and Data Engineering},
	publisher    = {IEEE},
	volume       = 18,
	number       = 1,
	pages        = {63--77}
}

@inproceedings{yang2022inducing,
	title        = {Inducing Neural Collapse in Imbalanced Learning: Do We Really Need a Learnable Classifier at the End of Deep Neural Network?},
	author       = {Yang, Yibo and Chen, Shixiang and Li, Xiangtai and Xie, Liang and Lin, Zhouchen and Tao, Dacheng},
	year         = 2022,
	booktitle      = {Proceedings of Neural Information Processing Systems (NeurIPS)},
	volume       = 35,
	pages        = {37991--38002}
}

@inproceedings{lin2017focal,
	title        = {Focal loss for dense object detection},
	author       = {Lin, Tsung-Yi and Goyal, Priya and Girshick, Ross and He, Kaiming and Doll{\'a}r, Piotr},
	year         = 2017,
	booktitle    = {Proceedings of the IEEE International Conference on Computer Vision (ICCV)},
	pages        = {2980--2988}
}

@inproceedings{cui2019class,
	title        = {Class-balanced loss based on effective number of samples},
	author       = {Cui, Yin and Jia, Menglin and Lin, Tsung-Yi and Song, Yang and Belongie, Serge},
	year         = 2019,
	booktitle    = {Proceedings of the IEEE/CVF Conference on Computer Vision and Pattern Recognition (CVPR)},
	pages        = {9268--9277}
}

@inproceedings{peifeng2023feature,
	title        = {Feature Directions Matter: Long-Tailed Learning via Rotated Balanced Representation},
	author       = {Peifeng, Gao and Xu, Qianqian and Wen, Peisong and Yang, Zhiyong and Shao, Huiyang and Huang, Qingming},
	year         = 2023,
	booktitle    = {International Conference on Machine Learning (ICML)},
	organization = {PMLR}
}

@inproceedings{wang2020long,
	title        = {Long-tailed Recognition by Routing Diverse Distribution-Aware Experts},
	author       = {Wang, Xudong and Lian, Long and Miao, Zhongqi and Liu, Ziwei and Yu, Stella},
	year         = 2021,
	booktitle    = {International Conference on Learning Representations (ICLR)}
}

@inproceedings{menon2021longtail,
	title        = {Long-tail learning via logit adjustment},
	author       = {Aditya Krishna Menon and Sadeep Jayasumana and Ankit Singh Rawat and Himanshu Jain and Andreas Veit and Sanjiv Kumar},
	year         = 2021,
	booktitle    = {International Conference on Learning Representations (ICLR)}
}

@article{cui2023generalized,
	title        = {Generalized parametric contrastive learning},
	author       = {Cui, Jiequan and Zhong, Zhisheng and Tian, Zhuotao and Liu, Shu and Yu, Bei and Jia, Jiaya},
	year         = 2023,
	journal      = {IEEE Transactions on Pattern Analysis and Machine Intelligence},
	publisher    = {IEEE}
}

@inproceedings{son2025dbm,
  title={Difficulty-aware balancing margin loss for long-tailed recognition},
  author={Son, Minseok and Koo, Inyong and Park, Jinyoung and Kim, Changick},
  booktitle={Proceedings of the AAAI Conference on Artificial Intelligence},
  volume={39},
  pages={20522--20530},
  year={2025}
}

@article{zhang2026e2igb,
  title={{E2IGB}: Enhanced effective-information-guided class-balanced loss for long-tailed object recognition},
  author={Zhang, Hong and He, Tian and Li, Zhigang and Deng, Yangyan and Feng, Yachun and Yuan, Ding and Yang, Yifan},
  journal={Pattern Recognition},
  volume={173},
  pages={112756},
  year={2026},
  publisher={Elsevier}
}

@article{zhang2025systematic,
  title={A systematic review on long-tailed learning},
  author={Zhang, Chongsheng and Almpanidis, George and Fan, Gaojuan and Deng, Binquan and Zhang, Yanbo and Liu, Ji and Kamel, Aouaidjia and Soda, Paolo and Gama, Jo{\~a}o},
  journal={IEEE Transactions on Neural Networks and Learning Systems},
  volume={36},
  number={8},
  pages={13670--13690},
  year={2025},
  publisher={IEEE}
}

@inproceedings{jian2026currmix,
  title={{CurrMix}: Curriculum-Enhanced {MixUp} for Long-Tailed Visual Recognition},
  author={Jian, Zhongquan and Chen, Yanhao and Hu, Bingbing and Lv, Wenhan and Wang, Shaopan and Wu, Jipeng and Yao, Junfeng and Lu, Yang and Wu, Qingqiang},
  booktitle={Proceedings of the IEEE/CVF Conference on Computer Vision and Pattern Recognition (CVPR)},
  pages={7706--7716},
  year={2026}
}

@inproceedings{hu2026parameter,
	title        = {Parameter manifold purification},
	author       = {Hu, Jiacong and Wu, Jinxun and Zhang, Shengxuming and Liu, Shunyu and Zhang, Haofei and Song, Mingli and Feng, Zunlei},
	year         = 2026,
	booktitle    = {Proceedings of the 43rd International Conference on Machine Learning (ICML)},
	series       = {Proceedings of Machine Learning Research},
	url={https://openreview.net/forum?id=TScgrE8pMG}
}

@inproceedings{alexandridis2025fractal,
  title={Fractal calibration for long-tailed object detection},
  author={Alexandridis, Konstantinos Panagiotis and Elezi, Ismail and Deng, Jiankang and Nguyen, Anh and Luo, Shan},
  booktitle={Proceedings of the IEEE/CVF Conference on Computer Vision and Pattern Recognition (CVPR)},
  pages={15139--15150},
  year={2025},
  organization={IEEE}
}

@article{shi2026lift+,
  title={{LIFT+}: Lightweight fine-tuning for long-tail learning},
  author={Shi, Jiang-Xin and Wei, Tong and Li, Yu-Feng},
  journal={IEEE Transactions on Pattern Analysis and Machine Intelligence},
  year={2026},
  publisher={IEEE}
}

@inproceedings{cortes2025improved,
  title={Improved balanced classification with theoretically grounded loss functions},
  author={Cortes, Corinna and Mohri, Mehryar and Zhong, Yutao},
  booktitle={Advances in Neural Information Processing Systems (NeurIPS)},
  volume={38},
  pages={77761--77802},
  year={2025}
}

\clearpage
\appendix
\section*{Supplementary Material for Generalised Balanced Softmax}
\phantomsection\label{app:gbs_supp}
\addcontentsline{toc}{section}{Supplementary Material}
\numberwithin{equation}{section}
\numberwithin{figure}{section}
\numberwithin{table}{section}

\noindent These appendices provide the proofs and supplementary results supporting the main text.

\section{Imbalanced gradients}\label{app:ib_grad}
Under the initial-state assumptions in \thmref{thm:ib_grad}, $\Delta_{ci}^{\note{grad}}$ is proportional to $\card{\set{C}}\card{\set{N}_{c}}-\card{\set{N}}$.
\begin{proof}
    \begin{gather}
       \Delta^{\note{grad}}_{ci}
       = \sum_{n\in \set{N}_{c}}(1-p_{nc}){h}_{ni} - \sum_{d\in \set{C}\backslash \{c\}}\sum_{n\in \set{N}_{d}} p_{nc}{h}_{ni} \\
       = \upsilon(\frac{\card{\set{N}_{c}}(\card{\set{C}}-1)}{\card{\set{C}}} - \frac{\card{\set{N}} - \card{\set{N}_{c}}}{\card{\set{C}}}) \\
       \propto \card{\set{C}}\card{\set{N}_{c}} - \card{\set{N}_{c}}   - \card{\set{N}} + \card{\set{N}_{c}}  \\
       \propto  \card{\set{C}}\card{\set{N}_{c}}   - \card{\set{N}}.
    \end{gather}
\end{proof}

\section{Proof for \thmref{thm:two_p}}\label{app:proof_thm_two_p}
The proof follows the class-prior adjustment identity used by~\citet{Ren2020balms}.
\begin{proof}
    Define the input marginals associated with the two class priors as
    \begin{equation}
        p_{\pi}(x_n)=\sum_{d\in\set{C}}p(x_n\mid y=d)\pi_d
        \quad\text{and}\quad
        p_{\hat{\pi}}(x_n)=\sum_{d\in\set{C}}p(x_n\mid y=d)\hat{\pi}_d.
    \end{equation}
    By Bayes' rule,
    \begin{equation}
        \phi_{nc}=p_{\pi}(y=c\mid x_n)
        =\frac{p(x_n\mid y=c)\pi_c}{p_{\pi}(x_n)},
    \end{equation}
    and hence
    \begin{equation}
        p(x_n\mid y=c)=\frac{\phi_{nc}p_{\pi}(x_n)}{\pi_c}.
    \end{equation}
    Substituting this expression into the conditional probability under $\hat{\pi}$ gives
    \begin{align}
        \hat{\phi}_{nc}
        &=p_{\hat{\pi}}(y=c\mid x_n)\\
        &=\frac{p(x_n\mid y=c)\hat{\pi}_c}
        {\sum_{d\in\set{C}}p(x_n\mid y=d)\hat{\pi}_d}\\
        &=\frac{\frac{\phi_{nc}p_{\pi}(x_n)}{\pi_c}\hat{\pi}_{c}}{\sum_{d\in\set{C}} \frac{\phi_{nd}p_{\pi}(x_n)}{\pi_d} \hat{\pi}_d} \\
        &=\frac{\frac{\hat{\pi}_c}{\pi_c}\phi_{nc}}
        {\sum_{d\in\set{C}}\frac{\hat{\pi}_d}{\pi_d}\phi_{nd}}.
        \label{appeq:prior_adjustment}
    \end{align}
    If $\phi_{nc}$ is represented by the standard softmax probability
    \begin{equation}
        \phi_{nc}=\frac{e^{z_{nc}}}{\sum_{j\in\set{C}}e^{z_{nj}}},
    \end{equation}
    substitution into \eqnref{appeq:prior_adjustment} yields
    \begin{equation}
        \hat{\phi}_{nc}
        = \frac{\frac{\hat{\pi}_c}{\pi_c}\phi_{nc}}{\sum_{d\in\set{C}}\frac{\hat{\pi}_d}{\pi_d}\phi_{nd}}
        = \frac{\frac{\hat{\pi}_c}{\pi_c} \frac{e^{z_{nc}}}{\sum_{j\in\set{C}}e^{z_{nj}}} }{\sum_{d\in\set{C}}\frac{\hat{\pi}_d}{\pi_d} \frac{e^{z_{nd}}}{\sum_{j\in\set{C}}e^{z_{nj}}} }
        =\frac{\frac{\hat{\pi}_c}{\pi_c}e^{z_{nc}}}
        {\sum_{d\in\set{C}}\frac{\hat{\pi}_d}{\pi_d}e^{z_{nd}}}.
    \end{equation}
    The denominator is common to all classes for a fixed sample $n$. Therefore, the corresponding decision rule is
    \begin{equation}
        \argmax_c\left(z_{nc}+\log\frac{\hat{\pi}_c}{\pi_c}\right).
    \end{equation}
\end{proof}

\section{Model-induced class marginal of the balanced-data reference model}\label{app:p_y=c_proof}
This section proves \thmref{thm:p_y=c_pattern}.
\begin{proof}
    Let $\epsilon$ satisfy $0<\epsilon<(\card{\set{C}}-1)/\card{\set{C}}$. Under the theorem's equal-confidence assumption, the reference model assigns
    \begin{equation}
        \phi_{nc}=1-\epsilon
        \quad\text{for }n\in\set{N}_c,
    \end{equation}
    and distributes the remaining probability uniformly, so that
    \begin{equation}
        \phi_{nc}=\frac{\epsilon}{\card{\set{C}}-1}
        \quad\text{for }n\notin\set{N}_c.
    \end{equation}
    By definition, $\bar{\phi}_c$ is the average of $\phi_{nc}$ over all training samples. Among the $\card{\set{N}}$ samples, $\card{\set{N}_c}$ belong to class $c$ and each contributes $1-\epsilon$ to this average. The remaining $\card{\set{N}}-\card{\set{N}_c}$ samples belong to other classes and each contributes $\epsilon/(\card{\set{C}}-1)$. Therefore,
    \begin{equation}
        \begin{aligned}
        \bar{\phi}_c
        &=\frac{1}{\card{\set{N}}}\sum_{n\in\set{N}}\phi_{nc}\\
        &=\frac{1}{\card{\set{N}}}
        \left[
        \card{\set{N}_c}(1-\epsilon)
        +(\card{\set{N}}-\card{\set{N}_c})
        \frac{\epsilon}{\card{\set{C}}-1}
        \right].
        \end{aligned}
    \end{equation}
    Subtracting the empirical class frequency $\card{\set{N}_c}/\card{\set{N}}$ and collecting the terms in $\epsilon$ gives
    \begin{equation}
        \begin{aligned}
        \bar{\phi}_c-\frac{\card{\set{N}_c}}{\card{\set{N}}}
        &=\frac{1}{\card{\set{N}}}
        \left[
        -\card{\set{N}_c}\epsilon
        +(\card{\set{N}}-\card{\set{N}_c})
        \frac{\epsilon}{\card{\set{C}}-1}
        \right]\\
        &=\frac{\epsilon}{\card{\set{N}}(\card{\set{C}}-1)}
        \left[
        -(\card{\set{C}}-1)\card{\set{N}_c}
        +\card{\set{N}}-\card{\set{N}_c}
        \right]\\
        &=\frac{\epsilon}
        {\card{\set{N}}(\card{\set{C}}-1)}
        \left(\card{\set{N}}-\card{\set{C}}\card{\set{N}_c}\right).
        \end{aligned}
    \end{equation}
    Because the factor multiplying the term in parentheses is positive, the difference is negative when $\card{\set{C}}\card{\set{N}_c}>\card{\set{N}}$ and non-negative otherwise, which proves the result.
\end{proof}

\section{Model-induced class marginal of the imbalanced model}\label{app:hat_p_y=c_proof}
This section examines the relation between the empirical class frequency $\card{\set{N}_c}/\card{\set{N}}$ and the model-induced class marginal $\bar{\hat{\phi}}_c$ for a model trained on imbalanced data using standard softmax. Motivated by the empirical results in \secref{ssec:target_probs}, suppose that
\begin{equation}
    \hat{\phi}_{nc}=1-\epsilon_c
    \quad\text{for }n\in\set{N}_c,
\end{equation}
and that the remaining probability $\epsilon_d$ for a sample in class $d$ is distributed uniformly among the other $\card{\set{C}}-1$ classes. For the $\card{\set{N}_c}$ samples in class $c$, each target-class probability contributes $1-\epsilon_c$ to $\bar{\hat{\phi}}_c$. For each other class $d$, its $\card{\set{N}_d}$ samples each assign probability $\epsilon_d/(\card{\set{C}}-1)$ to class $c$. Therefore,
\begin{equation}
    \begin{aligned}
    \bar{\hat{\phi}}_c
    &=\frac{1}{\card{\set{N}}}
    \sum_{n\in\set{N}}\hat{\phi}_{nc}\\
    &=\frac{1}{\card{\set{N}}}
    \left[
    \card{\set{N}_c}(1-\epsilon_c)
    +\sum_{d\in\set{C}\backslash\{c\}}
    \card{\set{N}_d}\frac{\epsilon_d}{\card{\set{C}}-1}
    \right].
    \end{aligned}
\end{equation}
Subtracting the empirical class frequency and collecting the terms gives
\begin{equation}
    \begin{aligned}
    \bar{\hat{\phi}}_c-\frac{\card{\set{N}_c}}{\card{\set{N}}}
    &=\frac{1}{\card{\set{N}}}
    \left[
    -\card{\set{N}_c}\epsilon_c
    +\sum_{d\in\set{C}\backslash\{c\}}
    \card{\set{N}_d}\frac{\epsilon_d}{\card{\set{C}}-1}
    \right]\\
    &=\frac{
    \sum_{d\in\set{C}\backslash\{c\}}\card{\set{N}_d}\epsilon_d
    -(\card{\set{C}}-1)\card{\set{N}_c}\epsilon_c}
    {\card{\set{N}}(\card{\set{C}}-1)}.
    \end{aligned}
\end{equation}
Because the denominator is positive, the following statements are equivalent:
\begin{equation}
    \begin{aligned}
    \bar{\hat{\phi}}_c
    &>\frac{\card{\set{N}_c}}{\card{\set{N}}}\\
    &\Longleftrightarrow
    \sum_{d\in\set{C}\backslash\{c\}}\card{\set{N}_d}\epsilon_d
    -(\card{\set{C}}-1)\card{\set{N}_c}\epsilon_c>0\\
    &\Longleftrightarrow
    \sum_{d\in\set{C}\backslash\{c\}}\card{\set{N}_d}\epsilon_d
    >(\card{\set{C}}-1)\card{\set{N}_c}\epsilon_c.
    \end{aligned}
\end{equation}
Similarly, $\bar{\hat{\phi}}_c$ is lower than the empirical class frequency when the final inequality is reversed. The empirical tendency for $\epsilon_c$ to decrease as class frequency increases makes the first condition more plausible for head classes and the reverse condition more plausible for tail classes. This is a finite-data interpretation of the empirical pattern in \tabref{tab:p_hat_dist}, not a universal guarantee for every model or class.

\section{Probabilities of target classes}
\label{ssec:target_probs}

This section analyses how the training-data composition affects the target-class probabilities predicted by the model. We evaluate two ResNet50 models: one trained on ImageNet-1K and used as the balanced-data reference model, and another trained on ImageNet-LT using standard softmax.
\figref{appfig:target_prob} shows the mean target-class probabilities produced by both models on the ImageNet-LT training set. The reference model produces relatively consistent values across classes. In contrast, the model trained on ImageNet-LT assigns higher target-class probabilities to classes with more training samples, a pattern related to the imbalanced-confidence issue discussed by~\citet{pan2021model}.

\begin{figure}[!tp]
    \centering
    \includegraphics[width=\textwidth]{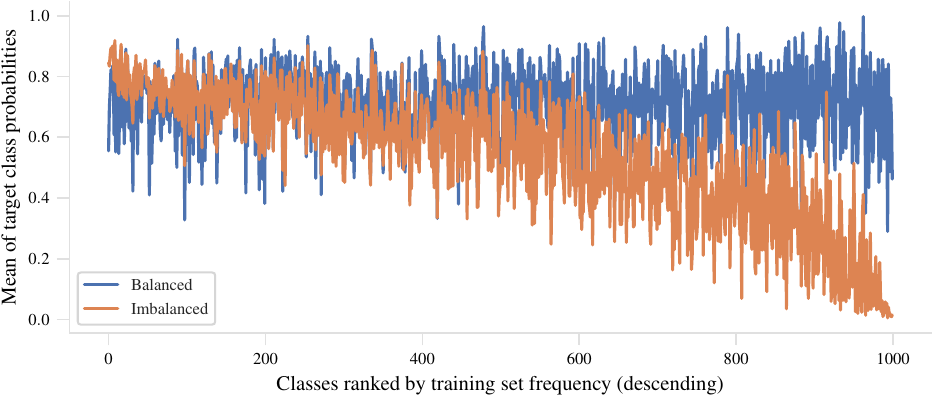}
    \caption{\textbf{Comparison of target class probabilities between
    balanced and imbalanced models.} The plot shows the mean target
    class probability calculated on the ImageNet-LT training data for
    both model types.}
    \label{appfig:target_prob}
\end{figure}

\section{Statistical results for long-tailed CIFAR benchmarks}\label{ssec:cifar_statistic}

\figref{fig:bs_gbs_cifar_box} illustrates the performance distribution across five independent runs for \gls{gbs} and Balanced Softmax. On the C10-LT benchmarks (imbalance factors 10, 50, and 100), \gls{gbs} achieves a statistically significant improvement in testing recall ($p < 0.05$). For the C100-LT benchmarks, while the $p$-values for C100-LT50 and C100-LT100 exceed the traditional significance threshold, the calculated Cohen’s $d$ ($d > 0.2$) suggests a small to medium effect size, indicating a practically meaningful trend toward improvement. In contrast, performance on C100-LT10 is comparable between the two methods. This convergence suggests that the advantages of \gls{gbs} are more pronounced as the data imbalance level increases, a trend also observed in the C10-LT results where the performance margin widens from the LT10 to the LT100 setting. These results suggest that \gls{gbs} is particularly effective in high-imbalance scenarios compared to Balanced Softmax.

\begin{figure}[!tp]
    \centering
    \includegraphics[width=\linewidth]{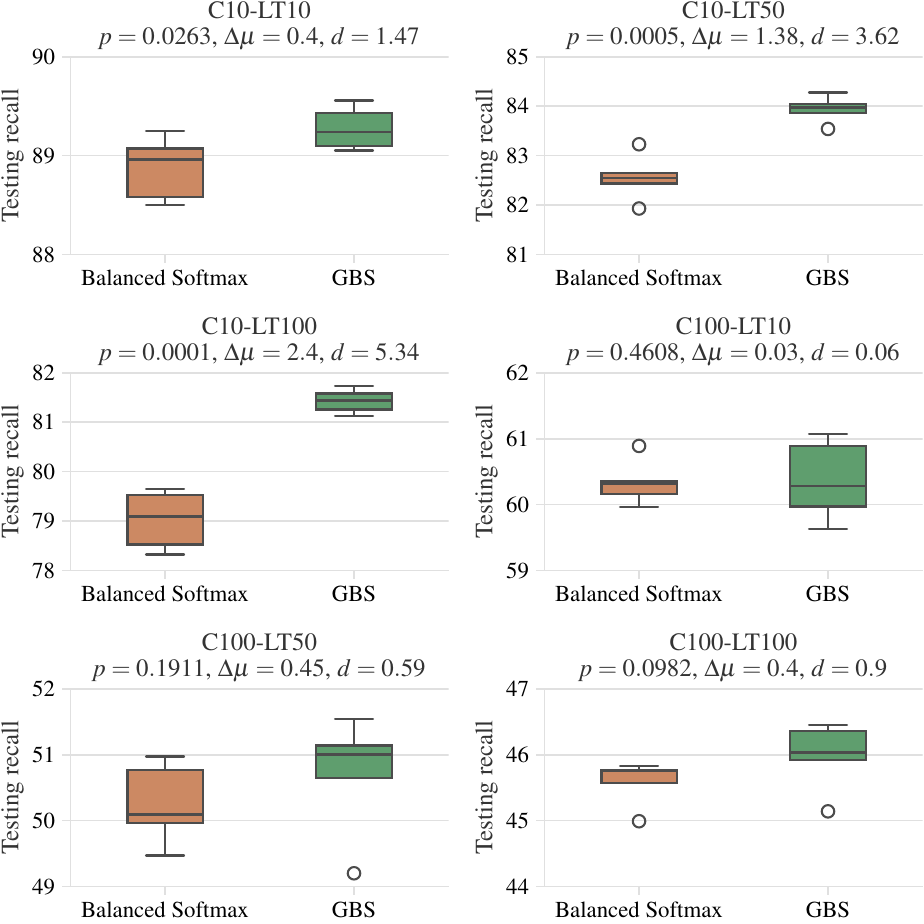}
    \caption{Statistical comparison of testing recall on
    long-tailed CIFAR datasets. The box plots illustrate the distribution
    of average testing recall across all classes over five independent
    runs for both Balanced Softmax and \gls{gbs}. Statistical significance, $p$-value, is determined via Welch’s $t$-test, where $\Delta\mu$ denotes the mean difference and $d$ represents Cohen’s $d$ effect size. }
    \label{fig:bs_gbs_cifar_box}
\end{figure}

\end{document}